\documentclass{article}

\usepackage{arxiv}
\usepackage{lineno,hyperref}

\usepackage{amssymb, amsmath, amsthm, amsfonts, mathrsfs}
\usepackage{graphicx}
\usepackage{url}
\usepackage{float}
\restylefloat{table}
\usepackage{multirow, pbox}
\usepackage{algorithm}
\usepackage{algorithmic}
\usepackage{caption}
\usepackage{subcaption}
\usepackage{longtable}
\usepackage{pdflscape}
\usepackage{hyperref}
\usepackage{listings}
\usepackage{geometry}
\usepackage{color}
\newtheorem{theorem}{Theorem}

\graphicspath{{../tnnls/}}

\begin{document}

\title{Learning Neural Network Classifiers with Low Model Complexity\thanks{This work has been submitted to the IEEE for possible publication. Copyright may be transferred without notice, after which this version may no longer be accessible.}}
\author{Jayadeva, Himanshu Pant, Mayank Sharma, Abhimanyu Dubey, \\
\textbf{Sumit Soman, Suraj Tripathi, Sai Guruju and Nihal Goalla}\\
Department of Electrical Engineering, Indian Institute of Technology, Delhi, India\\
\texttt{jayadeva@ee.iitd.ac.in}}

\maketitle

\begin{abstract}
Modern neural networks for large-scale learning tasks have very large model complexities. Many deep learning architectures have parameters that outnumber the dataset size. Recent contributions to deep learning techniques have focused on architectural modifications to improve parameter efficiency and performance. Most complexity bounds depend on the number of neurons and weights. In this paper, we derive a simple learning rule that leads to neural networks with improved generalization and low model complexities. We consider the last layer of a neural network as a classifier operating on a non-linear map generated by the previous layers, and obtain a continuous and differentiable upper bound on its fat shattering dimension. Using  backpropagation, we realize a training rule, termed as the LCNN learning rule, that minimizes the error on training samples, while keeping the bound on its model complexity small. The LCNN error functional also results in a higher value of the mean absolute gradient in the final and penultimate layers of CNNs/feedforward networks, considerably reducing the vanishing gradient problem. 
Our proposed approach yields benefits across a wide range of architectures, in comparison to and in conjunction with methods such as Dropout and Batch Normalization. Our results strongly suggest that deep learning techniques can benefit from model complexity control methods such as the LCNN learning rule.
\end{abstract}


\section{Introduction}\label{sec:introduction}

Deep neural networks have become an extremely popular learning technique, with significant deployment in a wide variety of practical domains such as computer vision \cite{muhammad2018efficient,kamel2018deep, yu2019reconstruction}, biosignal processing \cite{pourbabaee2017deep,king2018influence}, image captioning \cite{karpathy2015deep} and speech recognition \cite{hannun2014deep}. With the significant increase in dataset scale and subsequent increase in model complexity of multilayered neural network architectures, it has become imperative to learn networks that can offer performance guarantees, yield good generalization, and provide sparse representations. Vapnik's seminal work in computational learning theory \cite{vapnik1994measuring} highlighted that a small VC dimension and good generalization go hand in hand, however, minimizing the VC dimension as a function of the weights of the class of networks has remained elusive.

The representational redundancy in deep neural networks is well recognized. Some works on model complexity have recently appeared in literature \cite{xu2017kernel,yang2018clustering}. In many cases, the number of parameters exceeds the amount of training data resulting in severe overfitting \cite{zhang2016understanding}. Sontag \cite{sontag1998vc} derived that the VC dimension of neural network with $|W|$ weights is $O(|W|log(|W|))$, where $|W|$ is the cardinality of total number of weights. Hence, it is essential to reduce the redundancy in weights and neurons and enforce sparsity in the structure to bring down the VC dimension \cite{zhang2012discretized, bianchini2014complexity}. A number of methods have been proposed in the neural network domain to reduce model complexity \cite{han2015learning,sun2016sparsifying,wolfe2017incredible}, but these largely focus on pruning trained networks by removing synapses or neurons through heuristics or by applying sparsity inducing norms (e.g., $L_1$) and functions (e.g., Kullback-Leibler (KL) divergence).

In this paper, we introduce a novel approach to induce sparsity in neural networks by keeping the model complexity low. The approach, termed as the Low Complexity Neural Network (LCNN) in the following, uses a loss function that provides a tradeoff between model complexity and classification error. The model complexity is minimized via an upper bound on the fat shattering VC dimension of the network. Vapnik's risk formula provides a relation between the VC dimension and the structural risk; when the empirical risk gets minimized but the structural risk is not minimized, we obtain models that have low training error but poor generalization - which is known as overfitting. In formal terms, given a set $X$ of $M$ training samples in $n$ dimensions with labels $y$, the error rate on test samples, or the total risk, is bounded from above with confidence $(1 - \eta)$) by

\begin{gather}
  R(w)\leq R_{emp}(w) + \frac{\epsilon}{2} \left( 1+ \sqrt{1+\frac{4 R_{emp}(w)}{\epsilon}}\right)\label{eqn:vc1}
\end{gather}
where $\epsilon = 4\frac{h\left( ln \frac{2M}{h}+1\right) - ln (\eta)}{M}$, $h$ denotes the VC dimension, and $R_{emp}(w)$ denotes the empirical risk or error on the training samples. It is clear from the formula that small values of $h$ avoid overfitting by keeping the structural risk small. The LCNN rule achieves this objective by minimizing an upper bound on the fat shattering VC dimension. The rule is obtained by first deriving an upper bound on the VC dimension of the classifier layer of a neural network, and including this as part of the loss function to be minimized. The resultant loss function tries to minimize error on training samples, while keeping model complexity small. The gradient of this loss function leads to the LCNN learning rule. The models that are consequently learnt are sparse, yet provide good test set accuracies.

We demonstrate the consistent effectiveness of the LCNN rule across a variety of learning algorithms on various datasets across learning task domains. We see that the LCNN rule promotes higher test set accuracies, faster convergence and crisper, stronger feature representations across algorithms such as Feedforward (Fully Connected) Neural Networks (FNNs), Convolutional Neural Networks (CNNs) and Sparse Autoencoders (SAE), confirming our hypothesis that the algorithm indeed controls model complexity, while improving generalization performance.


\section{Related Work} \label{sec:relatedwork}
In recent machine learning research, we have seen significant interest in the area of model complexity control, and generalizability in deep neural networks. With the increase in the benchmark training data sizes and computational resource availability, modern large-scale learning tasks in computer vision and natural language processing have seen immensely large end-to-end neural network solutions \cite{krizhevsky2012imagenet, he2016deep}. The large number of parameters in these networks reduces generalizability in the learnt models, that also require large amounts of data and time to train.

Han et al. \cite{han2015learning} proposed a technique to induce sparsity in neural networks through iterative deletion of weights. Sparse formulations for regularization have also been employed to train networks that incorporate sparsity, such as in the work of Zhou et al. \cite{zhou2016less} and Scardapane et al. \cite{scardapane2016group}. Pruning has also been a studied domain for sparsification of neural networks, such as the work of Liu et al. \cite{liu2014pruning} that employed optimal brain damage to enforce sparsity, or the works of Srinivas et al. \cite{srinivas2015data} and Wolfe et al. \cite{wolfe2017incredible}, which pruned neurons based on weight-set saliency and second-order Taylor information respectively. Neural correlations have also been employed to identify neurons that can be removed from a network, in the works of Sun et al.\cite{sun2016sparsifying}, and Babaeizadeh et al. \cite{babaeizadeh2016noiseout}. Aghasi et al. \cite{aghasi2016net} employed sparse matrix transformations that maintained reconstructability to remove redundant neurons from networks.

On the other hand, we find that significant work has been done in improving generalization in deep neural networks as well. Arguably, the most popular technique to improve generalization in deep networks is Dropout \cite{srivastava2014dropout}, which attempts to prevent co-adaptation of neuronal activations. It uses an approximate model averaging technique that mutes a randomly selected set of activations during training. In a similar approach, Shakeout \cite{kang2016shakeout} enhances or suppress the neuron activations. DropConnect \cite{wan2013regularization} randomly sets weights to zero during training.

Cogswell et al. proposed DeCov \cite{cogswell2015reducing}, that encourages non-redundant representations in deep neural networks by minimizing the cross covariances of hidden activations. Injection of annealed gradient noise proposed in Neelakantan et al. \cite{neelakantan2015adding}, has also been shown to improve the validation performance of neural networks. Gabriel et al. \cite{pereyra2017regularizing} explored the regularization of neural networks by penalizing low entropy output distributions. Batch Normalization \cite{ioffe2015batch} normalizes the output of each neuron with the mean and standard deviation of the outputs calculated over each minibatch. The effect is reduction is covariate shift of the neuron outputs, bringing the Fisher matrix closer to identity matrix thereby, helping the networks to converge faster, and often with better validation performance.

Our contributions in this paper address both sections of deep learning research. By minimizing an upper bound on the VC dimension, we simultaneously achieve better generalization performance as well as sparser, more compact, explainable neural networks.

\section{The Low Complexity Neural Network \label{sec:LCNN}}
In the previous sections we described an introduction to our formulation, and provided relevant contemporary work in the domain. In this section, we describe the formulation and derivation of the low-complexity neural network (LCNN) learning rule. We largely develop ideas around \textbf{Theorem \ref{th4}}. Theorems prior to this have been included to provide context and also to help highlight the contribution of this work. We follow the notation of Sontag \cite{sontag1998vc}. Throughout the paper, we assume that the architecture of a neural network being considered is fixed, and that the weights are being trained by an algorithm such as backpropagation, in which weights are updated one layer at a time, proceeding from the output layer back to the input layer. This implies that when the weights of the ouput layer neurons are being updated, the weights of all previous layers are fixed. Furthermore, the class of an input sample is determined by thresholding the output of the neuron in the last layer. This permits us to treat the network as a binary classifier.

Feedforward neural networks are usually trained with some form of weight regularization as part of the loss function being optimized. In the context of a classifier, this impacts the margin of the classifier formed by the output layer neuron acting on the image vectors produced by the penultimate layer. We use this framework to draw upon the large body of literature on large margin classifiers. In particular, we use results on the fat shattering dimension to realize a simple learning rule that can significantly alter the generalization and sparsity of the trained network.

A specific assignment of weights yields a map $f$ from the input space $\mathscr{U} \rightarrow \left\{ 0, 1 \right\}$ (denotes the labels in a binary classification setting); the set of all functions $f$ forms a \textit{function class} $\mathscr{F}$. To each $f \in \mathscr{F}$ Sontag  \cite{sontag1998vc} associates the set $C_f = \{ u \in \mathscr{U}|f(u) = 1 \}$. A concept class $C_{\mathscr{F}} \equiv \{ C_f, f \in \mathscr{F} \}$ is associated with $\mathscr{F}$ and the VC dimension ($\gamma$) of $\mathscr{F}$ is defined as the supremum ($sup$) of the cardinality ($card$) of the set $S$, subject to $S$ being shattered by the concept class $C_{\mathscr{F}}$. This can be written as
\begin{gather}
  \gamma(\mathscr{F}) \equiv \left\{ sup \; card(S) ~|~ S \text{ is shattered by } C_{\mathscr{F}} \right\}
\end{gather}

Again, following \cite{sontag1998vc}, we consider a parameterized class of functions $C_{\mathscr{F}({w})}$, that correspond to functions generated on a fixed or specified architecture by varying the weights $w$ of the network. Note that the literature has many results where the function class is parameterized in terms of the number of weights. Such results have also been reported in Sontag's work, but such results are not of interest in the context of this paper; the concept class in such a case may be written as $C_w$.
Furthermore, given a set of real valued functions $\mathscr{F}$, the VC dimension of this family is defined as
\begin{gather}	\
  \gamma(\mathscr{F}) \equiv \gamma \left( \left\{ \mathscr(H) \circ f, f \in \mathscr{F} \right\} \right)
\end{gather}
where $H$ denotes the Heaviside function, which is $1$ when a concept class exists, and is zero otherwise. Further, $\circ$ denotes the composition operation. This implies that a if subset is shattered for a specific class assignment, there must exist $f\in \mathscr{F}$ having that assignment.

\begin{theorem}
Consider a single layer neural network classifier with $n$ inputs and one bias term. The VC dimension $\gamma$ of this network is bounded from above by $(n + 1)$.
\end{theorem}\label{th1}

\begin{proof}
As stated in \cite{sontag1998vc}, for a single layer network, the set of affine functions are linearly parameterized by vectors $(a_0,a_1,...,a_n)\in \Re^{n+1}$ as
\begin{gather}
f(u)=f(u_1,u_2,...,u_n)=a_0+a_1 u_1+...+a_n u_n
\end{gather}
Therefore, it can be seen that
\begin{gather}
\gamma(\mathscr{F})=n+1
\end{gather}
where $\mathscr{F}$ is the concept class and $\gamma$ is the VC dimension.
\end{proof}

We now consider a neural network classifier with $k$ layers and one output neuron, where all neurons have real-valued weights and a real-valued biases. Let the neural network have a total of $n_1$ neurons till the penultimate layer, which has $n$ neurons. The final layer of the neural network learns a  hyperplane classifier from the feature representation learnt by all layers before it. In other words, for an input sample, all layers except the penultimate layer act as a feature extractor, and the prediction is done at the last layer. The VC dimension of all layers excluding the final layer is a function $\theta(\cdot)$ of the number of neurons $n_1$. 
From Theorem 1, we know that the VC dimension of the final layer $\gamma$ is bounded from above by $(n+1)$. Hence, the VC dimension $\Theta$ of the network will be a monotonically increasing function of $(n_1 + n)$ \cite{baum1989size}, where $\gamma(n)\leq (n+1)$.



We now consider the classifier constructed with the output neuron to be a large margin classifier. In most conventional neural networks used today, the objective function being optimized at the output layer contains a regularization term such as $\| w \|^2$, where $w$ is the vector of weights of the output neuron. Following the context of the above discussion, we note that the margin of the classifier constructed by thresholding the output layer neuron would be given by $\frac{\left|V_{min}\right|}{\|w \|}$,
where $V_{min}$ is the minimum value of the output layer neuron (prior to applying any threshold to determine the class of an input sample) across all input training samples. The motivation behind examining the margin of our classifier is to use results from the  work \cite{batra2015learning}, that can be used to obtain bounds on the VC dimension.

\begin{theorem}
The VC dimension $\gamma$ 
satisfies

\begin{equation}\label{eqnha}
  \gamma \leq 1 + \operatorname{Min} \left( \frac{4R^2}{d_{min}^2}, n \right)
\end{equation}
where $R$ denotes the radius of the smallest sphere enclosing $\left\{ V^i_{(k - 1), 1}, V^i_{(k - 1), 2},..., V^i_{(k - 1), n} \right\}$, $i = 1, 2, ..., M$.
\end{theorem}\label{th4}

\begin{proof}
Vapnik \cite{vapnik98} showed that the VC dimension $\gamma$ for a family of fat margin hyperplane classifiers with a strictly positive margin $d$ of at least $d_{min}$ satisfies (\ref{eqnha}), where $R$ is the radius of the smallest hypersphere enclosing all the training samples. Consider the classifier formed by the output layer neuron acting on a weighted sum of the inputs $\{ V^i_{(k - 1), 1}, V^i_{(k - 1), 2},..., V^i_{(k - 1), n}\}$, $i = 1, 2, ..., M$. By definition, the classifier has a margin of at least $d_{min}$. The result follows.
\end{proof}

\begin{theorem}
For the neural network, let the set of input samples be denoted by $x^i$, $i = 1, 2, ..., M$.
The outputs of the penultimate layer neurons when the $i-th$ sample is presented at the input, are denoted by $\{ V^i_{(k - 1), 1}, V^i_{(k - 1), 2},..., V^i_{(k - 1), n}\}$. The weight vector of the output neuron is denoted by $w$, and its bias input is denoted by $b$. In other words, when the $i-th$ training sample is presented at the input, the net input to the final layer neuron is given by
\begin{gather}
  net^i = \sum_{j = 1}^n w_j ~ V^i_{(k-1),j} \;+\; b =  \mathbf{w}^T \mathbf{V}^i_{(k-1)} + b,
\end{gather}
and the class predicted by the classifier is determined from the sign of $net^i$. Further, assume that any member of the family of classifiers so defined has a fat margin of at least $d_{min}$, where $d_{min}$ is strictly positive. Then, 
$\gamma$ 
is bounded from above by
\begin{gather}
  \gamma \leq 1 + \operatorname{Min} \left( C\sum_{i=1}^M (net^i)^2 \text{ , } \; n \right), \label{ubound1}
\end{gather}
where $C$ is a real positive constant.
\end{theorem}\label{th5}

\begin{proof}
As defined earlier, let $R$ denote the radius of the smallest hypersphere enclosing all the inputs to the classifier neuron, i.e. $R$ is the radius of the set $\left\{ V^i_{(k - 1), 1}, V^i_{(k - 1), 2},..., V^i_{(k - 1), n}, \right\}$, $i = 1, 2, ..., M$. Since the activation functions and the input samples $x^i$ are bounded, the radius of the set $\left\{ V^i_{(k - 1), 1}, V^i_{(k - 1), 2},..., V^i_{(k - 1), n}, \right\}$, $i = 1, 2, ..., M$, given by $\textit{R}$, is also bounded.
Without loss of generality, we assume that the hyperplane
\begin{gather}
  \sum_{j = 1}^n w_j ~ V_{(k-1), j} + b \equiv \mathbf{w}^T \mathbf{V_{(k-1)}} + b = 0
\end{gather}
passes through the origin. Note that this is the decision surface corresponding to the output neuron. The requirement for this hyperplane to pass through the origin may be trivially ensured by the following construction: add one more neuron in the penultimate layer, whose output is always fixed at 1, and let the weight connecting its output to the final layer neuron be $b$. Let the outputs of the penultimate layer neurons be denoted by
\begin{gather}
  \mathbf{u}^i_{k-1} \equiv \left\{ V^i_{(k - 1), 1}, V^i_{(k - 1), 2},..., V^i_{(k - 1), n}, 1\right\}   \text{,} \nonumber \\
  i = 1, 2, ..., M.
\end{gather}
The set of samples $\mathbf{u}^i_{k-1}, i = 1, 2, ..., M$ constitute the inputs to the classifier. The radius of this set is denoted by $R_1$, and by definition, is given by
\begin{gather}
 R_1 = \operatorname*{Max}_{i = 1, 2, \ldots,M} \|\mathbf{u}^i_{k-1}\|.
\end{gather}

Further, we denote
\begin{gather}
  \mathbf{\beta} \equiv \left\{\mathbf{w}, b \right\}.
\end{gather}
The net input to the classifier neuron when the $i-th$ sample is presented at the input is therefore given by
\begin{gather}
  net^i = \mathbf{\beta}^T\mathbf{u}^i_{k-1} \equiv \mathbf{w}^T \mathbf{V}^i_{(k-1)} + b.
\end{gather}

Then, the margin, which is the distance of the closest point from the separating hyperplane, is given by -
\begin{gather}
d = \operatorname*{Min}_{i = 1, 2, \ldots,M} \frac{\|\mathbf{\beta}^T\mathbf{u}^i_{k-1}|}{\|\mathbf{\beta}\|} \label{geomargin0} \\
\frac{R_1}{d} = \frac{R_1}{\operatorname*{Min}_{i = 1, 2, \ldots,M} \frac{\|\mathbf{\beta}^T\mathbf{u}^i_{k-1}\|}{\|\mathbf{\beta}\|}} \\
= ~ \frac{R_1 ~ \|\mathbf{\beta}\|}{\operatorname*{Min}_{i = 1, 2, \ldots,M} \|\mathbf{\beta}^T\mathbf{u}^i_{k-1}\|}  \label{Rbyd1}
\end{gather}

Since the margin of the classifier is assumed to be at least $d_{min}$, the distance of any sample from the hyperplane is at least $d_{min}$, i.e.
\begin{gather}
  \frac{\|\mathbf{\beta}^T\mathbf{u}^i_{k-1}\|}{\|\mathbf{\beta}\|} \geq d_{min} \; \forall i = 1, 2, ..., M\\
 \implies \|\mathbf{\beta}\| \leq \frac{\|\mathbf{\beta}^T\mathbf{u}^i_{k-1}\|}{d_{min}} \; \forall i = 1, 2, ..., M\\
 \implies \|\mathbf{\beta}\|^2 \leq \frac{\|\mathbf{\beta}^T\mathbf{u}^i_{k-1}\|^2}{d_{min}^2} \; \forall i = 1, 2, ..., M
\end{gather}
Adding over all samples, we have
\begin{gather}
  \sum_{i = 1}^M ~ \|\mathbf{\beta}\|^2 \leq \sum_{i = 1}^M ~ \frac{\|\mathbf{\beta}^T\mathbf{u}^i_{k-1}\|^2}{d_{min}^2}\\
  M ~ \|\mathbf{\beta}\|^2 \leq \sum_{i = 1}^M ~ \frac{\|\mathbf{\beta}^T\mathbf{u}^i_{k-1}\|^2}{d_{min}^2}\\
  \implies \; \|\mathbf{\beta}\|^2 \leq \frac{1}{M} ~ \sum_{i = 1}^M ~ \frac{\|\mathbf{\beta}^T\mathbf{u}^i_{k-1}\|^2}{d_{min}^2}
\end{gather}
From (\ref{Rbyd1}), we have
\begin{gather}
  \frac{R_1^2}{d^2} \leq ~ \frac{R_1^2}{M ~ d_{min}^2} ~ \sum_{i = 1}^M ~ \frac{\left| \mathbf{\beta}^T\mathbf{u}^i_{k-1} \right |^2}{\operatorname*{Min}_{i = 1, 2, \ldots,M} \left|\mathbf{\beta}^T\mathbf{u}^i_{k-1}\right|^2}  \label{Rbyd2}
\end{gather}

We now use this bound in the context of a multi-layer feed-forward neural network. When training such a neural network, the target of the output neuron is usually chosen to be $t > 0$ (respectively, $-t$) for input patterns belonging to class $1$ (respectively, class $-1$); a typical value for $t$ may be $0.9$. Let us consider a set of patterns whose image vectors in the penultimate layer, are denoted by \textit{viz.} $\{ V^i_{(k - 1), 1}, V^i_{(k - 1), 2},..., V^i_{(k - 1), n}\}$, $i = 1, 2, ..., M$. Since the classifier realized by the final layer neuron shatters any data set whose size is smaller than its VC dimension $\gamma$, it follows that the set of samples under consideration, \textit{viz.}, $\{ V^i_{(k - 1), 1}, V^i_{(k - 1), 2},..., V^i_{(k - 1), n}\}$, with labels $y_i \in \left\{ -1, 1 \right \}$, $i = 1, 2, ..., M$, is linearly separable. For the trained network, we have
\begin{gather}
  f(net^i) \begin{cases}
    \geq t, & \text{if } y_i = 1\\
    \leq -t, & \text{if } y_i = -1
              \end{cases}
\end{gather}

where $net^i = \mathbf{w}^T \mathbf{V}_{k-1} + b \equiv \beta^T \mathbf{u}^i$, and where $y_i$ denotes the class label of the $i$-th sample. This may be written as

\begin{gather}
  net^i = \beta^T\mathbf{u}^i_{k-1} \; \begin{cases}
    \geq \theta, & \text{if } y_i = 1\\
    \leq -\theta, & \text{if } y_i = -1
              \end{cases} \nonumber\\
\implies ~ |\beta^T\mathbf{u}^i_{k-1}| \geq \theta > 0 \label{neti}\\
\implies \operatorname*{Min}_{i = 1, 2, \ldots,M} \left|\beta^T\mathbf{u}^i_{k-1} \right|^2 \geq \theta^2 \label{denomtheta}
\end{gather}

where $\theta = f^{-1}(t)$.\\

Since $f(\cdot)$ is usually a saturating non-linearity, we usually have $\theta \geq 1$, and for the sake of further discussion we will assume this to be the case. Note that $f$ and $f^{-1}$ are monotonically increasing functions of their arguments.
From (\ref{denomtheta}) and (\ref{Rbyd2}), we obtain  have
\begin{gather}
  \frac{R_1^2}{d^2} \leq ~ \frac{R_1^2}{M ~ d_{min}^2 ~\theta^2} ~ \sum_{i = 1}^M ~ \left| \beta^T\mathbf{u}^i_{k-1} \right |^2 \label{Rbyd3}\\
  \text{or } \frac{R_1^2}{d^2} \leq C' ~\sum_{i = 1}^M ~ \left| \beta^T\mathbf{u}^i_{k-1} \right |^2,
\end{gather}
where
\begin{gather}\label{Cprime}
  C' = \frac{R_1^2}{M ~ d_{min}^2 ~\theta^2}.
\end{gather}
Substituting values from (\ref{Rbyd3}) and (\ref{Cprime}) into (\ref{eqnha}), we have that the VC dimension of the final layer classifier, $\gamma$, satisfies
\begin{gather}
  \gamma \leq 1 + \operatorname{Min} \left( C \sum_{i = 1}^M ~ \left| \beta^T\mathbf{u}^i_{k-1} \right |^2, n\right) \text{, }\; \text{where } C = 4 C'.
\end{gather}
which may be written more succinctly as
\begin{gather}\label{VCbound}
  \gamma \leq 1 + \operatorname{Min} \left( C \sum_{i = 1}^M ~ \left( net^i \right)^2, n\right) \text{, }\; \text{where } C = 4 C'.
\end{gather}
\end{proof}

\subsection{Bound on Neural Network Parameters}

In a classical neural network, commonly used loss functionals for measuring the error at the output layer neuron are
\begin{itemize}
	\item For binary classification with labels $y_i \in \{-1,1\}$  or $y_i \in \{0,1\}$, the empirical error $E_{emp}$ is given as $E_{emp} =   \frac{1}{2M} \sum_{i=1}^M \left( y_i - f(net^i) \right)^2 $
	      where $M$ denotes the number of training samples and $f(net^i)$ denotes the activation function $f(\cdot)$ on net input to neuron $i$ given by $net^i$.
	\item For multiclass classification with $y_i \in \{0,K-1\}$ for $K$ classes $E_{emp} =  \frac{1}{M} \sum_{i=1}^M \left( -log(\operatorname{softmax}(net^i)) \right)$, where $softmax$ denotes the activation function computed as $\frac{e^{net^i}}{\sum_i e^{net^i}}$ and $log(\cdot)$ denotes the logarithmic operator.
\end{itemize}
\begin{itemize}
 \item
In regularized neural networks, this is modified by adding a term proportional to $\|w\|^2$ at individual weights (weight decay). In order to minimize model complexity, the LCNN uses the modified error functional $E$ for binary classification as $E =  E_{emp}+ C~\sum_{i=1}^M (net^i)^2$ and for multiclass as $E =  E_{emp}+ C ~\sum_{i=1}^M \sum_{j=1}^K  (net_j^i)^2$, where $net_j^i$ is the score of $j^{th}$ class for the $i^{th}$ sample.
\end{itemize}

\subsection{Application of the VC Bound on Hidden Layers}
The application of the bound derived in (\ref{ubound1}) to the pre-activations in a net can be interpreted as a $L_2$ regularizer on them, since it forces pre-activations to be close to zero. For ReLU activation functions $max(0,x)$, our data dependent regularizer forces the pre-activations for each layer to be close to zero. Thus, it in turn enforces sparsity at neuronal levels in the intermediate layers. The gain of most saturating activation functions is maximum when the net input is close to zero; as we see in the following, the LCNN training rule reduces the vanishing gradient problem.

We minimize the error with explicit $L_2$ regularizers on weights along with the model complexity term (\ref{th1}). 
Consider a feedforward architecture with $k-1$ hidden layers. For an intermediate layer $h$, let the activations of the layer $h-1$ with $l_{h-1}$ neurons be $\mathbf{u}^i_{h-1} \in \Re^{l_{h-1}}$. Let $\mathbf{w}_{h_i} \in \Re^{l_{h-1}},\,\, \forall \,\, i \in \{1, \ldots, l_h\}$ be the weights of the layer $h$ going from $h-1$ to $h$ and $b_{h_i}$ be the set of biases.
This is equivalent to the minimization problem given by (\ref{eqnh15}).

\begin{gather}
\operatorname{Min} \frac{1}{2}\sum_{j=1}^{l_h}\|\mathbf{w}_{h_j}\|_2^2 +
\frac{D}{2} \sum_{i=1}^{M}\sum_{j=1}^{l_h} (\mathbf{w}_{h_j}^T \mathbf{u}_{h-1}^i+b_{h_j})^2 \label{eqnh15}
\end{gather}
Applying (\ref{eqnh15}) to all the hidden layers along with the classifier layer, we achieve the following loss function that tries to minimize the overall complexity of neural network in terms of its synapses and neuronal activations.
\begin{gather}
	E = E_{emp} + \frac{C}{2}\sum_{h=0}^{k-1}\sum_{j=1}^{l_h}\|\mathbf{w}_{h_j}\|_2^2 + \frac{C}{2}\sum_{j=1}^{K}\|\mathbf{w}_{c_j}\|_2^2 +\nonumber \\
	  \frac{D}{2} \sum_{i=1}^{M}\sum_{l=0}^{k-1}\sum_{j=1}^{l_h} (\mathbf{w}_{h_j}^T \mathbf{u}_{h-1}^i + b_{h_j})^2 +
	\frac{D}{2} \sum_{i=1}^{M}\sum_{j=1}^{K} (net_j^i)^2 \label{LCNN_all_multiclass}
\end{gather}

The two regularizers ($L_2$ and model complexity control) enforce regularized synapses and promote sparse activity in neurons. When related to the recent work of Gabriel et~al. \cite{pereyra2017regularizing}, where the authors propose penalizing confident distributions by applying a penalty to prevent overfitting, we see that our model complexity term applied to classifier layer tries to achieve a similar effect.

A highly confident softmax classifier might have larger scores for the correct class of the point as opposed to the low scores for incorrect classes, if the class of the point is correctly predicted. The term $\sum_{i=1}^M \sum_{j=1}^K  (net_j^i)^2$, reduces such effects, by preventing any of the scores from being large. In turn, it prevents overfitting as is evident from the experiments.

\section{Experiments}\label{sec:emperical}

To evaluate the effectiveness of our formulation in reducing model complexity and improving regularization across various architectures, we conduct several qualitative and quantitative experiments. We categorize our experiments based on dataset sizes for convenience to the reader. This section initially presents results obtained on small datasets, followed by results on larger datasets using multiple network architectures where the LCNN loss functional is employed.

\subsection{Experiments on UCI datasets}

Table \ref{tab:ucidataset} summarizes the description of these datasets in terms of number of attributes, samples and missing attribute values. 
K-Nearest Neighbor (KNN) imputation method was used for handling missing attribute values as it is robust to bias between classes in the dataset. The datasets was randomly shuffled and sampled into five sets. Four sub-sets of these (80\% of the total data) were used as the training set and remaining (20\%) was used as the validation set. Accuracies have hence been obtained using 5-fold cross validation scheme. This process was repeated 10 times to remove the effect of randomization. This is shown in Fig. \ref{fig:exptproc}. For the case of multi-class (more than 2 classes) classification, target sets were prepared by 1-of-$K$ encoding scheme where $K$ is number of classes and ${K > 2}$. The number of neurons in output layer is equal to $K$ where (${K > 2}$); for the case of two classes, single neuron was used in the output layer.

\begin{figure}[hbtp]
\centering
\includegraphics[scale=0.5]{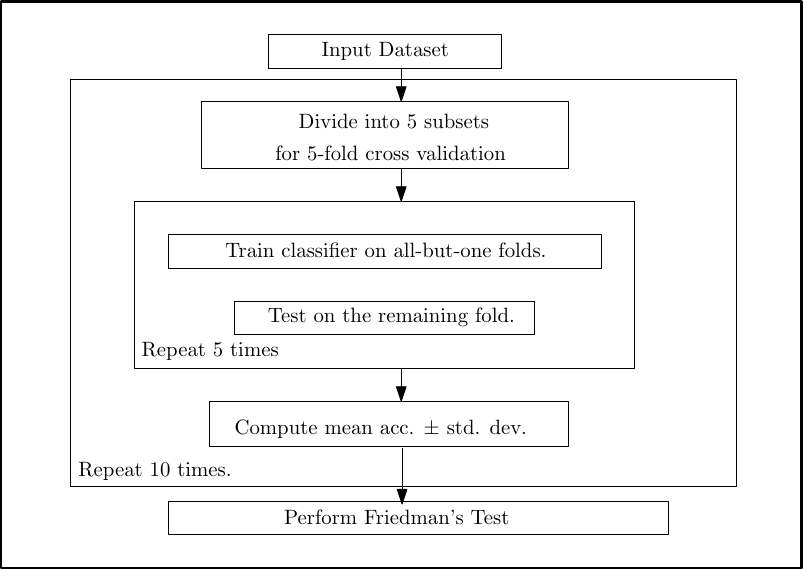}
\caption{Flowchart for experimental procedure.}
\label{fig:exptproc}
\end{figure}

\begin{table}[htbp]
 \centering
 \caption{Summary of UCI Benchmark Datasets}
 \scalebox{0.65}{
    \begin{tabular}{|c|c|c|c|c|}
    \hline
    \textbf{S.No.} & \textbf{Dataset} & \textbf{Size x Feat} & \textbf{Classes } & \textbf{Missing Values} \\
    \hline
    1     & Pimaindians & 768x4 & 2     & No \\
    2     & Heartstat & 270x13 & 2     & No \\
    3     & Haberman & 306x3 & 2     & No \\
    4     & Hepatitis & 155x19 & 2     & No \\
    5     & Ionosphere & 351x34 & 2     & No \\
    6     & Transfusion & 748x4 & 2     & No \\
    7     & ECG   & 132x12 & 2     & No \\
    8     & voting & 435x16 & 2     & No \\
    9     & Fertility & 100x9 & 2     & No \\
    10    & Australian & 690x14 & 2     & No \\
    11    & crx   & 690x15 & 2     & No \\
    12    & mamm-masses & 961x5 & 2     & No \\
    13    & german & 1000x20 & 2     & No \\
    14    & plrx  & 182x12 & 2     & No \\
    15    & sonar & 208x60 & 2     & No \\
    16    & housevoters & 436x16 & 2     & No \\
    17    & balance & 576x4 & 2     & No \\
    18    & Wholesale & 440x7 & 2     & No \\
    19    & glass & 214x10 & 6     & No \\
    20    & seed  & 210x7 & 3     & No \\
    21    & blogger & 100x5 & 2     & No \\
    22    & IPLD  & 583x10 & 2     & No \\
    23    & teaching assistant & 151x5 & 3     & No \\
    24    & iris  & 150x4 & 3     & No \\
    25    & zoo   & 101x16 & 7     & No \\
    26    & letter & 20000x16 & 26    & No \\
    27    & hayes & 160x5 & 3     & No \\
    28    & Breast Cancer Wisconsin & 699x9 & 2     & Yes \\
    29    & HeartSpectf & 267x44 & 2     & No \\
    30    & Horse & 368x27 & 2     & Yes \\
    31    & Sensorless Drive & 58509x49 & 11    & No \\
    32    & MiniBooNE & 130064x50 & 2     & No \\
    \hline
    \end{tabular}%
    }
 \label{tab:ucidataset}%
\end{table}%

Training set was scaled between $-1$ and $1$. Target set values were kept at $+1$ (for class $1$) and $-1$ (for class $-1$). For multiclass classification one versus rest approach \cite[pp. 182, 338]{bishop2006pattern} was used based on activations of the output neurons.

For the UCI datasets, LCNN with one hidden layer and Tan-Hyperbolic transfer function was used. Number of neurons in the hidden layer and value of C was optimized by grid search. We compared the performance of LCNN with other state of the art algorithms, namely SVM-Linear and Kernel versions (LIBSVM Package) \cite{chang2011libsvm}, MCM Linear and Kernel versions \cite{jayadeva2015learning}, Feed Forward Neural Networks (FFNN) with regularization \cite{bebis1994feed}
, based on average test accuracy and standard deviation using five fold cross validation. For Linear  MCM and SVM implementations, we used soft margin classifier and for kernel versions, Gaussian kernel with soft margin was used. FFNN was implemented with one hidden layer using Tan-Hyperbolic transfer function. 
A summary of these methods, alongwith the hyperparameters for each model, are shown in Table \ref{tab:methods}. In all cases hyperparameters were optimized using grid search. 

\begin{table}[htbp]
  \centering
  \caption{Summary of methods with which the LCNN is compared}
  \scalebox{0.73}{
    \begin{tabular}{|c|c|c|}
    \hline
    \textbf{S.No.} & \textbf{Method} & \textbf{Hyperparameters} \\
    \hline
    1     & NNMCM & C (Complexity Term), Hidden Layer Neurons \\
    2     & NN-Reg & Reg Parameters, Hidden Layer Neurons \\
    3     & SVM-Lin & C (Soft Margin Term) \\
    4     & SVM-Ker & C (Soft Margin Term), Width of Gaussian  \\
    5     & MCM-Lin & C (Soft Margin Term) \\
    6     & MCM-Ker & C (Soft Margin Term), Width of Gaussian \\
    \hline
    \end{tabular}%
    }
  \label{tab:methods}%
\end{table}%

\subsubsection{Test Accuracies on UCI Datasets}

Table \ref{tab:UCIRes} shows the test accuracies obtained by all the algorithms considered on the UCI Benchmark datasets. Results are shown as Mean Accuracy $\pm$ Standard Deviation for each dataset. The best performing model in terms of accuracy and standard deviation are indicated in boldface. From  \ref{tab:UCIRes}, it can be inferred that performance of LCNN is better that other algorithms for most of the datasets, followed by Kernel MCM.

\begin{table*}[hbtp]
 \centering 
  \caption{Classification accuracies on the UCI Datasets}
    \scalebox{0.65}{
    \begin{tabular}{|c|c|c|c|c|c|c|c|}
    \hline
    \textbf{S.No.} & \textbf{Dataset (samples X features X classes)} & \multicolumn{1}{|c|}{\textbf{SVM(Linear)}} & \multicolumn{1}{|c|}{\textbf{SVM(Kernel)}} & \multicolumn{1}{|c|}{\textbf{NN-Regularization}} & \multicolumn{1}{|c|}{\textbf{LCNN}} & \multicolumn{1}{|c|}{\textbf{MCM(Linear)}} & \multicolumn{1}{|c|}{\textbf{MCM(Kernel)}} \\
    \hline
    \hline
    1     & Pimaindians (768x4x2) & 76.5  $\pm$ 2.99 & 76.81  $\pm$ 3.96 & 76.11  $\pm$ 3.60 & \textbf{77.97  $\pm$ 2.69} & 76.95  $\pm$ 2.79 & 65.81  $\pm$ 0.32 \\
    2     & Heartstat (270x13x2) & 83.33 $\pm$ 4.71 & 83.33  $\pm$ 5.39 & 81.01  $\pm$ 4.82 & \textbf{85.55  $\pm$ 4.61} & 84.44 $\pm$ 3.60 & 84.07  $\pm$ 1.29 \\
    3     & Haberman (306x3x2) & 72.22 $\pm$ 1.17 & 72.32  $\pm$ 1.18 & 73.11  $\pm$ 2.71 & \textbf{75.46 $\pm$ 1.13} & 72.54 $\pm$ 1.35 & 73.86 $\pm$ 1.53 \\
    4     & Hepatitis (155x19x2) & 80.00  $\pm$ 6.04 & 82.64  $\pm$ 4.60 & 81.11  $\pm$ 6.29 & \textbf{85.16  $\pm$ 7.43} & 80.71  $\pm$ 2.60 & 84.23  $\pm$ 4.50 \\
    5     & Ionosphere (351x34x2) & 87.82  $\pm$ 2.11 & 88.87  $\pm$ 2.74 & 86.21  $\pm$ 4.28 & \textbf{92.87 $\pm$ 2.86} & 88.30  $\pm$ 3.72 & 86.31  $\pm$ 4.14 \\
    6     & Transfusion (748x4x2) & 76.20  $\pm$ 0.27 & 76.60  $\pm$ 0.42 & 76.01  $\pm$ 1.57 & \textbf{79.35  $\pm$ 1.89} & 76.89  $\pm$ 0.27 & 77.89  $\pm$ 1.00 \\
    7     & ECG (132x12x2) & 84.90  $\pm$ 5.81 & 85.65  $\pm$ 5.37 & 86.25  $\pm$ 6.64 & \textbf{91.22  $\pm$ 6.20} & 87.92  $\pm$ 6.18 & 89.43  $\pm$ 7.30 \\
    8     & Fertility (100x9x2) & 85.03  $\pm$ 6.03 & 88.03  $\pm$ 2.46 & 87.91  $\pm$ 6.51 & \textbf{88.93  $\pm$ 2.46} & 88.03  $\pm$ 2.46 & 88.03  $\pm$ 2.46 \\
    9    & Australian (690x14x2) & 85.50 $\pm$ 4.04 & 85.64  $\pm$ 4.24 & 85.24  $\pm$ 3.52 & 87.97  $\pm$ 3.10 & 85.79  $\pm$ 0.88 & 88.1  $\pm$ 3.32 \\
    10    & Credit Approval (690x15x2) & 69.56  $\pm$ 0 & 69.56   $\pm$ 0 & 68.14  $\pm$ 0.94 & \textbf{70.3  $\pm$ 0} & 69.56   $\pm$ 0 & 69.56   $\pm$ 0 \\
    11    & Mamm-masses (961x5x2) & 78.87  $\pm$ 2.14 & 79.91  $\pm$ 3.02 & 77.96  $\pm$ 2.00 & 81.16  $\pm$ 3.00 & 81.21  $\pm$ 4.01 & 81.58  $\pm$ 2.37 \\
    12    & German Credit (1000x20x2) & 74.1  $\pm$ 2.77 & 73.60  $\pm$ 1.19 & 75.8  $\pm$ 2.88 & 76.6  $\pm$ 2.39 & 74.20  $\pm$ 2.72 & 77.87  $\pm$ 2.53 \\
    13    & Planning Relax (182x12x2) & 71.44  $\pm$ 1.06 & 71.44  $\pm$ 1.06 & 71.05  $\pm$ 3.54 & \textbf{71.99  $\pm$ 1.94} & 71.44  $\pm$ 1.06 & 71.01  $\pm$ 1.01 \\
    14    & SONAR (208x60x2) & 76.02  $\pm$ 6.70 & 78.38  $\pm$ 7.67 & 86.62 $\pm$ 6.90 & 87.10 $\pm$ 5.73 & 75.98  $\pm$ 3.99 & 88.48  $\pm$ 5.45 \\
    15    & House Votes (435x16x2) & 95.88 $\pm$ 1.90 & 96.10  $\pm$ 1.87 & 95.56 $\pm$ 1.56 & \textbf{97.02 $\pm$ 1.00} & 96.10 $\pm$ 1.92 & 96.82  $\pm$ 1.00 \\
    16    & Balance (576x4x2) & 94.61  $\pm$ 1.68 & 98.43  $\pm$ 1.13 & 97.39  $\pm$ 2.39 & \textbf{98.78  $\pm$ 0.98} & 96.13  $\pm$ 1.78 & 97.84  $\pm$ 1.87 \\
    17    & Wholesale (440x7x2) & 89.54 $\pm$ 1.88 & 87.27  $\pm$ 2.81 & 91.07  $\pm$ 2.25 & 92.05  $\pm$ 0.64 & 91.14 $\pm$ 1.83 & 92.65  $\pm$ 1.56 \\
    18    & Glass (214x10x6)  & 92.11  $\pm$ 3.08 & 93.10  $\pm$ 1.85 & 92.42  $\pm$ 2.27 & \textbf{96.21  $\pm$ 3.22} & 93.75  $\pm$ 1.88 & 95.90  $\pm$ 2.08 \\
    19    & Seed (210x7x3) & 92.71  $\pm$ 4.65 & 95.00  $\pm$ 5.41 & 93.14  $\pm$ 10.0 & 95.16  $\pm$ 1.39 & 95.01 $\pm$ 3.19 & 96.01 $\pm$ 3.19 \\
    20    & Blogger (100x5x2) & 70.93  $\pm$ 12.4 & 80.10  $\pm$ 8.07 & 79.50  $\pm$ 9.35 & 80.1  $\pm$ 8.07 & 74.84  $\pm$ 13.5 & 87.87  $\pm$ 7.92 \\
    21    & IPLD (583x10x2) & 71.35  $\pm$ 0.39 & 71.35  $\pm$ 0.39 & 71.05  $\pm$ 4.20 & \textbf{73.85  $\pm$ 3.77} & 71.35  $\pm$ 0.09 & 72.97  $\pm$ 1.87 \\
    22    & Teaching Assistant (151x5x3) & 64.47  $\pm$ 12.0 & 68.63  $\pm$ 6.84 & 70.68  $\pm$ 8.53 & \textbf{74.68  $\pm$ 5.42} & 67.17  $\pm$ 9.27 & 71.68  $\pm$ 8.48 \\
    23    & Iris (150x4x3) & 96.44  $\pm$ 3.47 & 97.26  $\pm$ 2.87 & 97.33  $\pm$ 2.42 & 97.33  $\pm$ 0.94 & 96.12  $\pm$ 2.88 & 97.91  $\pm$ 1.08 \\
    24    & Zoo (101x16x7) & 96.4   $\pm$ 4.50 & 90.78   $\pm$ 5.65 & 96.10  $\pm$ 1.80 & \textbf{97.94  $\pm$ 1.74} & 95.55  $\pm$ 2.01 & 95.87  $\pm$ 2.34 \\
    25    & Letter  (20000x16x26) & 84.21  $\pm$ 0.89 & 82.23  $\pm$ 1.05 & 79.93  $\pm$ 3.01 & \textbf{87.08  $\pm$ 4.98} & 85.11  $\pm$ 2.25 & 85.34  $\pm$ 3.87 \\
    26    & Hayes Roth (160x5x3) & 60.43  $\pm$ 10 & 60.40  $\pm$ 5 & 75.23  $\pm$ 3.47 & 75.18  $\pm$ 2.48 & 61.25  $\pm$ 8.37 & 66.32  $\pm$ 9.13 \\
    27    & Breast Cancer Wisconsin (699x9x2)* & 96.6  $\pm$ 1.9 & 96.5  $\pm$ 1.3 & 94.83  $\pm$ 2.07 & \textbf{96.83  $\pm$ 0.77} & 96.32  $\pm$ 1.08 & 96.01  $\pm$ 0.96 \\
    28    & Heart Spectf (267x44x2) & 78.89  $\pm$ 1.02 & 79.16  $\pm$ 1.23 & 79.03  $\pm$ 1.17 & \textbf{81.79  $\pm$ 2.01} & 79.01  $\pm$ 1.56 & 79.03  $\pm$ 1.02 \\
    29    & Horse (368x27x2)* & 84.01  $\pm$ 5.76 & 84.52  $\pm$ 3.76 & 83.84  $\pm$ 1.49 & \textbf{86.83  $\pm$ 1.97} & 86.18  $\pm$ 4.49 & 87.05  $\pm$ 3.01 \\
    30    & Sensorless Drive (58509x49x11) & 90.01  $\pm$ 0.34 & 91.37  $\pm$ 0.75 & 95.35  $\pm$ 3.49 & \textbf{98.98  $\pm$ 1.07} & 93.27  $\pm$ 2.42 & 96.52  $\pm$ 1.11 \\
    31    & MiniBooNE (130064x50x2) & 85.71  $\pm$ 2.49 & 86.11  $\pm$ 3.12 & 85.35  $\pm$ 3.49 & \textbf{89.98  $\pm$ 1.07} & 85.11  $\pm$ 1.39 & 86.66  $\pm$ 0.87 \\
 \hline
\multicolumn{8}{c}{*-Datasets have missing attributes.}\\
    \end{tabular}%
    }
  \label{tab:UCIRes}%
\end{table*}

\subsubsection{Training Time on UCI Datasets}

Table \ref{tab:UCIResTime} shows the training time in seconds for the various approaches compared against the LCNN. This comparison is significant in order to establish the scalability of the LCNN in terms of training time vis-\`a-vis other approaches. The time indicated in \ref{tab:UCIResTime} is shown in the format Mean Training Time $\pm$ Standard Deviation across the training folds of the respective datasets. From the results, one can see that the LCNN scales well for large datasets as compared to linear and kernel versions of SVM and MCM. 


\begin{table*}[hbtp]
 \centering 
  \caption{Training Time for the UCI Datasets}
     \scalebox{0.65}{
    \begin{tabular}{|c|c|c|c|c|c|c|c|}
    \hline
      \textbf{S.No.} & \textbf{Dataset} & \multicolumn{1}{|c|}{\textbf{SVM(Linear)}} & \multicolumn{1}{|c|}{\textbf{SVM(Kernel)}} & \multicolumn{1}{|c|}{\textbf{NN-Regularization}} & \multicolumn{1}{|c|}{\textbf{LCNN}} & \multicolumn{1}{|c|}{\textbf{MCM(Linear)}} & \multicolumn{1}{|c|}{\textbf{MCM(Kernel)}}  \\
    \hline
    \hline
 
    1     & Pimaindians (768x4x2) & 0.021 $\pm$ 0.13 & 0.025 $\pm$ 0.006 & 0.18 $\pm$ 0.03 & \textbf{0.169 $\pm$ 0.07} & 2.0 $\pm$ 0.66 & 7.38 $\pm$ 0.15 \\
    2     & Heartstat (270x13x2) & \textbf{0.08 $\pm$ 0.002} & 0.12 $\pm$ 0.005 & 0.17 $\pm$ 0.08 & 0.14 $\pm$ 0.007 & 1.24 $\pm$ 0.2 & 2.01 $\pm$ 0.08 \\
    3     & Haberman (306x3x2) & \textbf{0.002 $\pm$ 0.004} & 0.004 $\pm$ 0.0008 & 0.12 $\pm$ 0.001 & 0.14 $\pm$ 0.01 & 0.91 $\pm$ 0.22 & 2.43 $\pm$ 0.19 \\
    4     & Hepatitis (155x19x2) & \textbf{0.008 $\pm$ 0.12} & 0.009 $\pm$ 0.006 & 0.14 $\pm$ 0.009 & 0.13 $\pm$ 0.01 & 0.70 $\pm$ 0.01 & 2.38 $\pm$ 0.50 \\
    5     & Ionosphere (351x34x2) & \textbf{0.01 $\pm$ 0.009} & 0.018 $\pm$ 0.009 & 0.17 $\pm$ 0.12 & 0.14 $\pm$ 0.005 & 2.7 $\pm$ 0.07 & 7.31 $\pm$ 0.89 \\
    6     & Transfusion (748x4x2) & \textbf{0.01 $\pm$ 0.002} & 0.022 $\pm$ 0.001 & 0.14 $\pm$ 0.01 & 0.12 $\pm$ 0.007 & 1.35 $\pm$ 0.08 & 8.11 $\pm$ 1.02 \\
    7     & ECG (132x12x2) & 0.007 $\pm$ 0.006 & \textbf{0.004 $\pm$ 0.005} & 0.16 $\pm$ 0.01 & 0.14 $\pm$ 0.03 & 0.29 $\pm$ 0.05 & 1.89 $\pm$ 0.11 \\
    8     & Fertility (100x9x2) & \textbf{0.007 $\pm$ 0.001} & 0.0084 $\pm$ 0.002 & 0.15 $\pm$ 6.51 & 0.16 $\pm$ 0.009 & 0.35 $\pm$ 0.025 & 1.71 $\pm$ 0.11 \\
    9    & Australian (690x14x2) & 0.03 $\pm$ 0.09 & \textbf{0.02 $\pm$ 0.007} & 0.16 $\pm$ 0.12 & 0.14 $\pm$ 0.01 & 2.89 $\pm$ 0.7 & 8.32 $\pm$ 1.01 \\
    10    & Credit Approval (690x15x2) & \textbf{0.03 $\pm$ 0.009} & 0.04  $\pm$ 0.005 & 0.13 $\pm$ 0.006 & 0.13 $\pm$ 0.005 & 3.01  $\pm$ 0.89 & 8.47  $\pm$ 1.19 \\
    11    & Mamm-masses (961x5x2) & \textbf{0.08 $\pm$ 0.004} & 0.09 $\pm$ 0.001 & 0.18 $\pm$ 0.04 & 0.15 $\pm$ 0.01 & 3.07 $\pm$ 0.52 & 10.31 $\pm$ 0.56 \\
    12    & German Credit (1000x20x2) & \textbf{0.07 $\pm$ 0.02} & 0.09 $\pm$ 0.01 & 0.16 $\pm$ 0.02 & 0.15 $\pm$ 0.01 & 5.27 $\pm$ 0.95 & 15.21 $\pm$ 1.82 \\
    13    & Planning Relax (182x12x2) & \textbf{0.016 $\pm$ 0.01} & 0.03 $\pm$ 0.004 & 0.13 $\pm$ 0.006 & 0.12 $\pm$ 0.004 & 0.60 $\pm$ 0.001 & 2.05 $\pm$ 0.17 \\
    14    & SONAR (208x60x2) & \textbf{0.08 $\pm$ 0.004} & 0.097 $\pm$ 0.005 & 0.16 $\pm$ 0.012 & 0.15 $\pm$ 0.02 & 1.87 $\pm$ 0.13 & 2.37 $\pm$ 0.49 \\
    15    & House Votes (435x16x2) & \textbf{0.008 $\pm$ 0.004} & 0.012 $\pm$ 0.03 & 0.14 $\pm$ 0.017 & 0.13 $\pm$ 0.007 & 1.6 $\pm$ 0.05 & 4.11 $\pm$ 0.72 \\
    16    & Balance (576x4x2) & \textbf{0.007 $\pm$ 0.001} & 0.009 $\pm$ 0.0012 & 0.16 $\pm$ 0.0.009 & 0.14 $\pm$ 0.001 & 1.12 $\pm$ 0.01 & 6.23 $\pm$ 0.91 \\
    17    & Wholesale (440x7x2) & \textbf{0.007 $\pm$ 0.0008} & 0.009 $\pm$ 0.0001 & 0.19 $\pm$ 0.14 & 0.18 $\pm$ 0.07 & 1.12 $\pm$ 0.01 & 5.87 $\pm$ 1.01 \\
    18    & Glass (214x10x6)  & \textbf{0.003 $\pm$ 0.001} & 0.007 $\pm$ 0.0013 & 0.32 $\pm$ 0.01 & 0.30 $\pm$ 0.02 & 0.85 $\pm$ 0.09 & 1.88 $\pm$ 0.1 \\
    19    & Seed (210x7x3) & \textbf{0.001 $\pm$ 0.00001} & 0.002 $\pm$ 0.0003 & 0.16 $\pm$ 0.017 & 0.18 $\pm$ 0.02 & 1.00 $\pm$ 0.1 & 1.56 $\pm$ 0.29 \\
    20    & Blogger (100x5x2) & \textbf{0.001 $\pm$ 0.0002} & 0.001 $\pm$ 0.0003 & 0.15 $\pm$ 0.05 & 0.14 $\pm$ 0.02 & 0.51 $\pm$ 0.08 & 1.02 $\pm$ 0.08 \\
    21    & IPLD (583x10x2) & \textbf{0.01 $\pm$ 0.0003} & 0.02 $\pm$ 0.00001 & 0.17 $\pm$ 0.06 & 0.17 $\pm$ 0.05 & 2.1 $\pm$ 0.23 & 7.71 $\pm$ 1.87 \\
    22    & Teaching Assistant (151x5x3) & \textbf{0.002 $\pm$ 0.0005} & 0.003 $\pm$ 0.0008 & 0.39 $\pm$ 0.02 & 0.35 $\pm$ 0.03 & 0.71 $\pm$ 0.11 & 1.32 $\pm$ 0.27 \\
    23    & Iris (150x4x3) & \textbf{0.01 $\pm$ 0.001} & 0.013 $\pm$ 0.003 & 0.17 $\pm$ 0.01 & 0.15 $\pm$  0.03 & 0.65 $\pm$ 0.01 & 1.24 $\pm$ 0.13 \\
    24    & Zoo (101x16x7) & \textbf{0.02 $\pm$ 0.0001}  & 0.03 $\pm$ 0.004 & 0.18 $\pm$ 0.01 & 0.18 $\pm$ 0.01 & 0.82 $\pm$ 0.13 & 1.89 $\pm$ 0.3 \\
    25    & Letter  (20000x16x26) & \textbf{11.5 $\pm$  0.05} & 21.09  $\pm$ 0.82 & 34.93 $\pm$ 4.21 & 27 $\pm$ 3.22 & 87 $\pm$ 3.2 & 187 $\pm$ 6.7 \\
    26    & Hayes Roth (160x5x3) & 0.008 $\pm$ 0.00001  & \textbf{0.001 $\pm$ 0.0005} & 0.23 $\pm$ 0.09 & 0.19 $\pm$ 0.03 & 0.71 $\pm$ 0.1 & 2.23 $\pm$ 0.5 \\
    27    & Breast Cancer Wisconsin (699x9x2)* & \textbf{0.03 $\pm$ 0.0009} & 0.05 $\pm$ 0.0005 & 0.32 $\pm$ 0.08 & 0.29 $\pm$ 0.06 & 2.1 $\pm$ 0.9 & 8.57 $\pm$ 0.51 \\
    28    & Heart Spectf (267x44x2) & \textbf{0.01 $\pm$ 0.0002} & 0.01 $\pm$ 0.0008 & 0.23 $\pm$ 0.04 & 0.18 $\pm$ 0.01 & 1.27 $\pm$ 0.1 & 2.52 $\pm$ 0.13 \\
    29    & Horse (368x27x2)* & \textbf{0.012 $\pm$ 0.001} & 0.012 $\pm$ 0.001 & 0.18 $\pm$ 0.008 & 0.12 $\pm$ 0.002 & 1.87 $\pm$ 0.21 & 2.93 $\pm$ 0.25 \\
    30    & Sensorless Drive (58509x49x11) & 135 $\pm$ 3.28 & 404 $\pm$ 8.91 & 127 $\pm$ 4.12 & \textbf{119 $\pm$ 2.16} & 325 $\pm$ 4.62 & 690 $\pm$8.48 \\
    31    & MiniBooNE (130064x50x2) & 1595 $\pm$ 24 & 2896 $\pm$ 35 & 227 $\pm$ 8.32 & \textbf{205 $\pm$ 3.32} & 2835 $\pm$ 32 & 4240 $\pm$ 45 \\

    \hline
\multicolumn{8}{c}{*-Datasets have missing attributes. Training time has been reported as mean $\pm$ std dev across the five folds.}\\
          
    \end{tabular}%
    }
  \label{tab:UCIResTime}%
  \end{table*}

To support above claim further, we compare the training time and classification accuracy with increasing number of samples for the MiniBooNE dataset for the various approaches vis-\`a-vis the LCNN to validate its scalability. In Fig. \ref{fig:ucitrainingtime}, the logarithm of the time in seconds is plotted along the vertical axis, while the increasing number of samples are shown along the horizontal axis. It may be noted here that of the 130,064 samples of the MiniBooNE dataset, 30,064 have been taken for testing and the scale-up for the LCNN has been shown by varying the training dataset size upto 100,000 samples. The objective here is to show how the LCNN scales with increasing  number of samples in the training dataset. It can be observed that the time taken by the LCNN is significantly lower than that of linear/kernel SVMs. It is close to the time taken by feed-forward neural networks with regularization. This affirms the scalability of the LCNN, which is one of the primary objectives of proposing this neural network architecture. The test set accuracy is consistently superior than competing approaches as shown in Fig. \ref{fig:ucitrainingaccuracy}, which indicates that the generalization ability of the LCNN scales well with dataset size.

\begin{figure}[hbtp]
    \centering
        \includegraphics[scale=0.23]{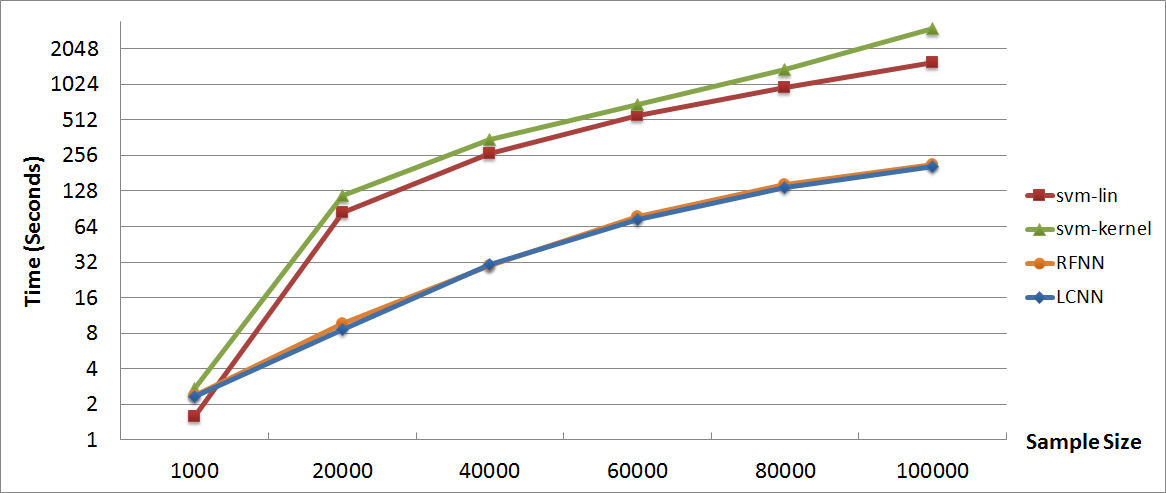}
        \caption{Effect of dataset size on LCNN training time.}
        \label{fig:ucitrainingtime}
\end{figure}

\begin{figure}[hbtp]
    \centering
        \includegraphics[scale=0.22]{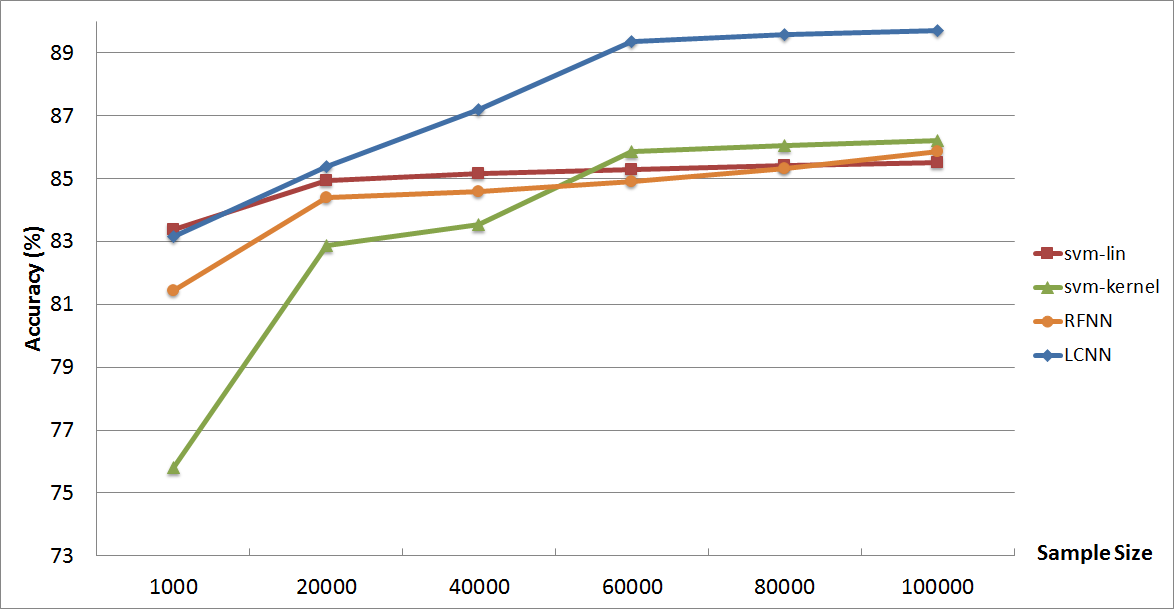}
        \caption{Effect of dataset size on LCNN accuracy.}
        \label{fig:ucitrainingaccuracy}
\end{figure}

\subsubsection{Comparison of approaches on UCI Datasets}

To verify that the results arrived at on UCI datasets are independent of the randomizations strategy that may affect the data distribution across the folds, we performed the Friedman's test \cite{friedman1937use}. The chi-squared value is \textbf{36.24} and the p-value is $6.5E-6$, which indicates that the results are not a consequence of the randomization technique.

We also present a comparative analysis of the performance of the LCNN on UCI benchmark datasets w.r.t. other comparative approaches in terms of p-values determined using Wilcoxon's signed ranks test \cite{wilcoxon1945individual}. The Wilcoxon Signed-Ranks Test is a measure of the extent of statistical deviations in the results obtained using a particular approach. A p-value less than $0.05$ indicates that the results of our approach have a significant statistical difference with the results obtained using the algorithms being compared, whereas p-values greater than $0.05$ indicate non-significant statistical difference.

The p-values for the approaches considered are shown in Table \ref{tab:UCIResPVals}. It can bee concluded that the algorithm works significantly better than linear SVM (``svm-lin''), kernel SVM (``svm-ker''), feed-forward neural network with regularization (``nn-Regularization''), linear MCM (``mcm-lin'') and kernel MCM (``mcm-ker'').

\begin{table}[htbp]
  \centering
  \caption{p-values for LCNN vs other approaches}
  \scalebox{0.7}{
    \begin{tabular}{|c|c|c|}
    \hline
    \textbf{S.No} & \textbf{Algorithm} & \textbf{p value} \\
    \hline
    1     & svm-lin & 7.95E-07 \\
    2     & svm-ker & 1.17E-06 \\
    3     & nn-Regularization & 1.30E-06 \\
    4     & mcm-lin & 8.75E-07 \\
    5     & mcm-ker & 2.16E-02 \\
    \hline
    \end{tabular}%
    }
  \label{tab:UCIResPVals}%
\end{table}%

\subsection{Experiments on large datasets}

\subsubsection{Description of large datasets}

The following datasets were used in this study.
\begin{itemize}

\item{MNIST:}
The MNIST \cite{lecun1998mnist} dataset isa collection of handwritten digits across 10 classes and 60,000 samples - 50,000 of which are reserved for a pre-defined training set and the remaining 10,000 are used as the test set. We split the training set in a ratio of 9:1 to construct a training set and validation set, since we require to tune the hyperparameters $C$ and $D$.

\item{CIFAR-10:}
CIFAR-10 \cite{krizhevsky2009learning} is a subset of 60,000 images taken from the much larger 80M tiny images dataset \cite{torralba200880}. Similar to MNIST, CIFAR-10 has 60,000 total samples spread across 10 classes equally, and a predefined train-test split of 5:1. Identical to MNIST, we split the training set in the ratio of 9:1 to create our own training and validation splits. Unlike MNIST, however, this dataset has a higher complexity, with $32\times32$ RGB images across 10 different object categories, as opposed to the grayscale $28\times28$ images in MNIST.

\item{ILSVRC12:}
The ImageNet Large Scale Visual Recognition Challenge \cite{deng2009imagenet} is an important benchmark in object classification. With the large dataset size and large number of classes, it is an accurate test for algorithm's scalability and performance at scale. The dataset has a total of 1.2M training images spread across 1000 object categories, and has a validation set of another 50,000 images. We compare several state of the art architectures like Alexnet \cite{krizhevsky2012imagenet}, VGGNet-19 \cite{simonyan2014very} and GoogLeNet \cite{szegedy2015going} on this dataset along with the interventions of complexity control. 

\end{itemize}

\subsubsection{Setup and Notation}
All our experiments are run on NVIDIA Tesla K40 GPU, and implementations were done using the assistance of the Caffe \cite{jia2014caffe} library. The notation used for simplicity in understanding experimental results is given in Table \ref{tab:dlnotation}.

\begin{table}[htbp]
	\centering
	\scalebox{0.7}{
	\begin{tabular}{|c|c|}
		\hline
		Symbols & Meaning \\
			\hline
		S     & Softmax \\
		SE    & Squared Error \\
		W     & $L_2$ regularization (weight-decay) \\
		LC-L  & LCNN applied only on last layer \\
		LC-A  & LCNN applied on all fully-connected layers \\
		D     & Dropout \\
		D-A     & Dropout on all fully-connected layers\\
		BN    & Batch normalization \\
			\hline
	\end{tabular}%
	\caption{Notations used.}
	\label{tab:dlnotation}
	}%
\end{table}%

First, we examine the improvement observed in validation performance across several image classification datasets, varying across popular deep CNN image classification architectures. Second, we turn our focus to fully connected feedforward neural networks where the problem of representational redundancy is the highest and finally we examine unsupervised learning using Sparse Autoencoders, where we demonstrate the effectiveness of model complexity regularizer viz., sparser and crisp set of weights. Our experiments span a variety of dataset sizes and architectures, and we describe these quantitative experiments in the following sections.

\subsubsection{Convolutional Neural Networks}\label{sec:CNN}

\textbf{ILSVRC12:} For our first set of experiments, we compare the original $L_2$ regularized architectures with the low-complexity (LCNN) error term applied to the classifier (final) layer on the ImageNet dataset. The results are summarized in Table \ref{tab:Imagenet_acc}. We observe that adding the model complexity term consistently improves the Top-5 error and Top-1 error across various architectures.
\begin{table}[t]
	\centering
	\scalebox{0.7}{
	\begin{tabular}{|l|c|c|c|}
		\hline
		Algorithm       & Top-5 Error & Top-1 Error\\
		\hline
		AlexNet        & 19.8     &   43.2\\
		AlexNet + LC-L   & \textbf{18.3}     &  \textbf{42.1} \\ \hline
		VGGNet-19      & 11.5     &  31.5 \\
		VGGNet-19 + LC-L & \textbf{10.2 }    & \textbf{30.2}  \\ \hline
		GoogLeNet      & 11.0     &  31.3 \\
		GoogLeNet + LC-L & \textbf{10.2}   & \textbf{29.8} \\
		\hline
	\end{tabular}%
	}
	\caption{Performance on ImageNet for the best hyperparameter settings (for GoogLeNet, LCNN was applied to all layers).}
	\label{tab:Imagenet_acc}
\end{table}
In Figure (\ref{fig:Imagenetconvergence}), we show the top-1 validation accuracy of AlexNet architecture for Imagenet dataset for various values of hyperparamter of LCNN term. We can see that for a certain near optimal value of the hyperparameter choice, we get a faster convergence, with a better value of accuracy.
\begin{figure}[t]
	\centering
	\includegraphics[scale=0.3]{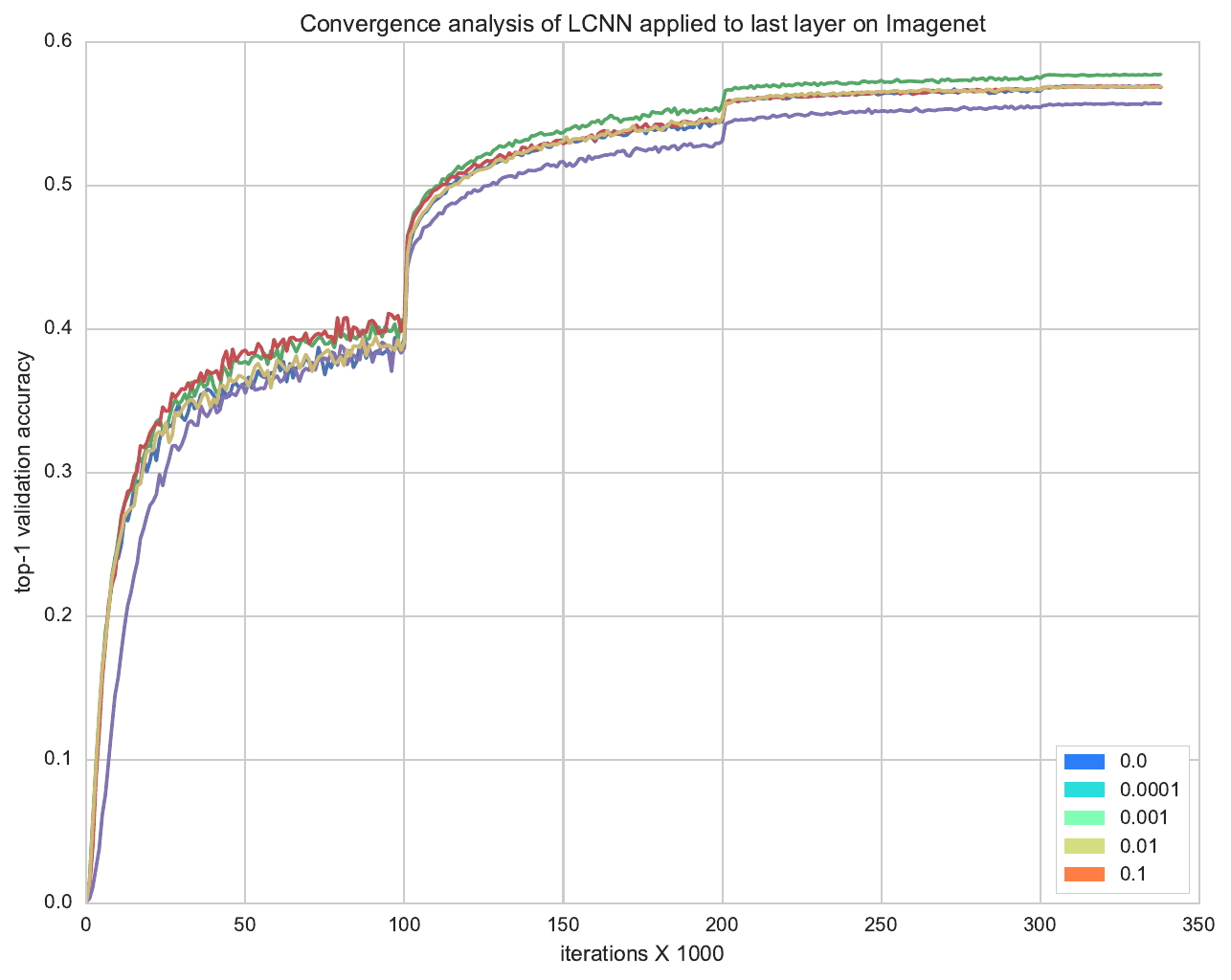}
	\caption{Convergence of AlexNet on ImageNet for different values of hyperparameter $C$ in the LCNN objective. The performance rises initially during training  and persists until convergence, indicating that the LCNN learns a good model early with few training samples.}
	\label{fig:Imagenetconvergence}
\end{figure}

\textbf{CIFAR 10 and MNIST:}  In our second set of experiments with CNNs, we compare the performance of the LCNN objective applied in combination with various other schemes like Dropout and Batch Normalization. For MNIST, we use Caffe's implementation of the LeNet architecture \cite{jia2014caffe}, and we employ Caffe's \cite{jia2014caffe} CIFAR-10 Quick architecture for CIFAR-10. LeNet consists of two convolutional layers interspersed with max-pooling layers, followed by two fully-connected layers.

The CIFAR-10 Quick architecture consists of three convolutional layers and two fully connected layers. Both these models were chosen for their quick training performance, enabling a proof of concept for our experiments. Our consistent improvements across a variety of architectures consolidate the impact of the LCNN objective in improving validation performance, and hence it can be applied to larger architectures as well with corresponding rises in performance.

For our experiments with Dropout, we follow a consistent protocol of applying it to  the penultimate fully-connected layer of each architecture (in case of D) or to all fully connected layers  (in case of D-A). We apply Batch Normalization (where mentioned) to all layers of each architecture. The hyperparameters $C$ and $D$ were tuned in the range $ [10^{-3},10]$ and $ [10^{-9},10^{-3}]$ respectively in multiples of $10$, and dropout percentage was kept at 30\%, 50\% and 70\%. Quantitative results for MNIST and CIFAR-10 are presented in Tables \ref{tab:mnist_acc} and \ref{tab:cifar10_acc} respectively.

\begin{figure*}[ht]
	\centering
	\begin{subfigure}[b]{0.43\textwidth}
		\centering
		\includegraphics[scale=0.095]{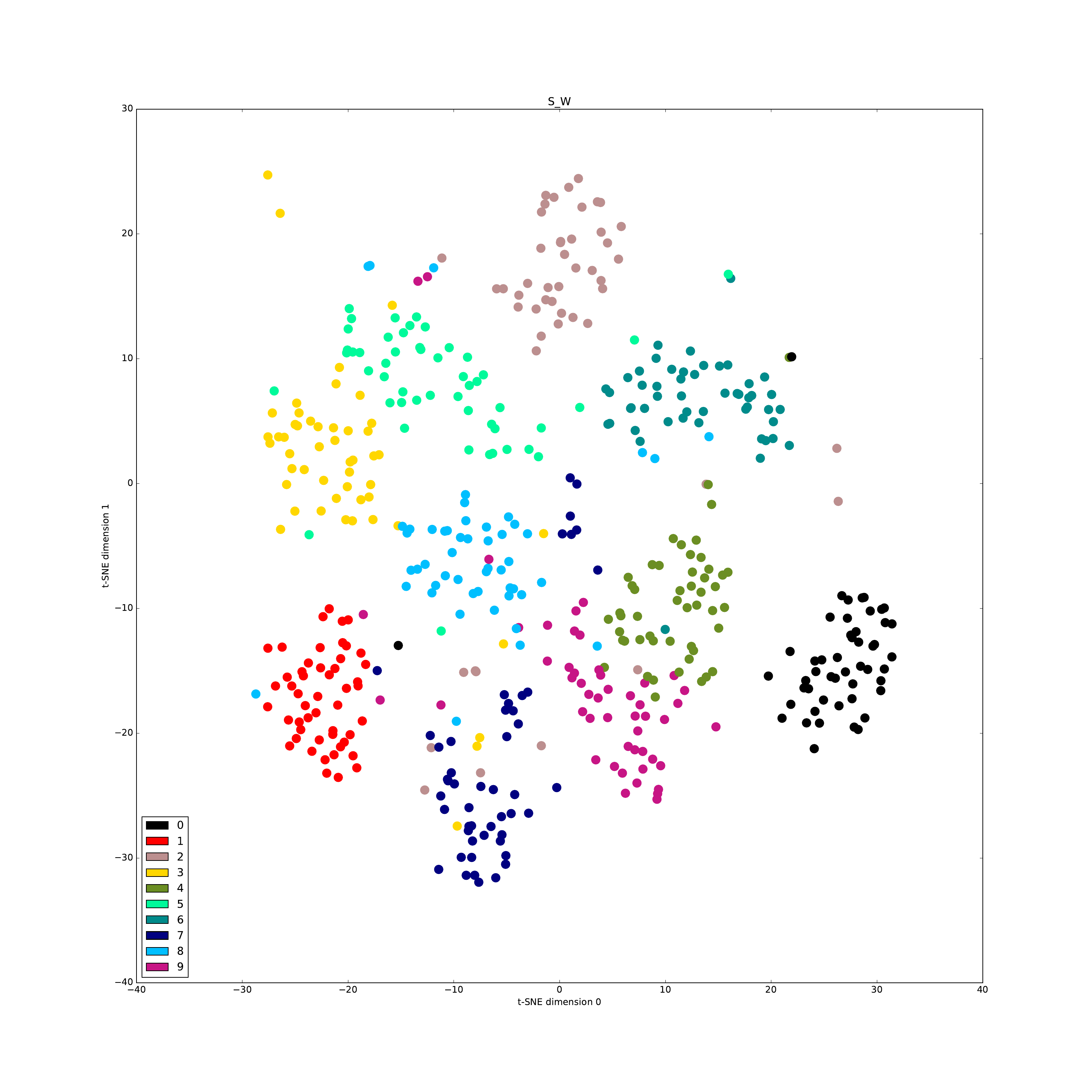}
		\caption{MNIST}
		\label{fig:mnistSW}
	\end{subfigure}
	~ 
	\begin{subfigure}[b]{0.43\textwidth}
		\centering
		\includegraphics[scale=0.095]{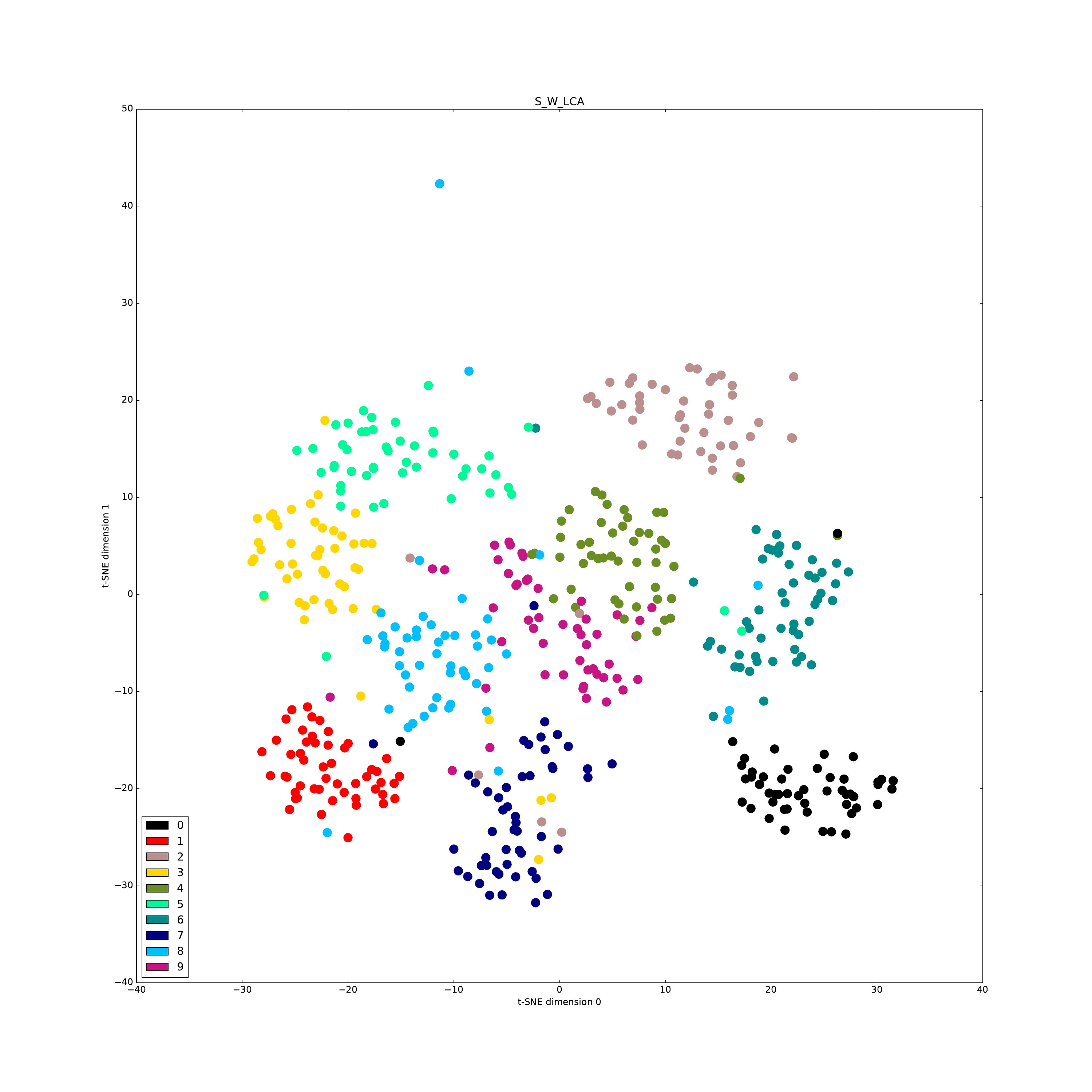}
		\caption{MNIST with LCNN}
		\label{fig:mnistSWLCA}
	\end{subfigure}
	
	\begin{subfigure}[b]{0.43\textwidth}
		\centering
		\includegraphics[scale=0.095]{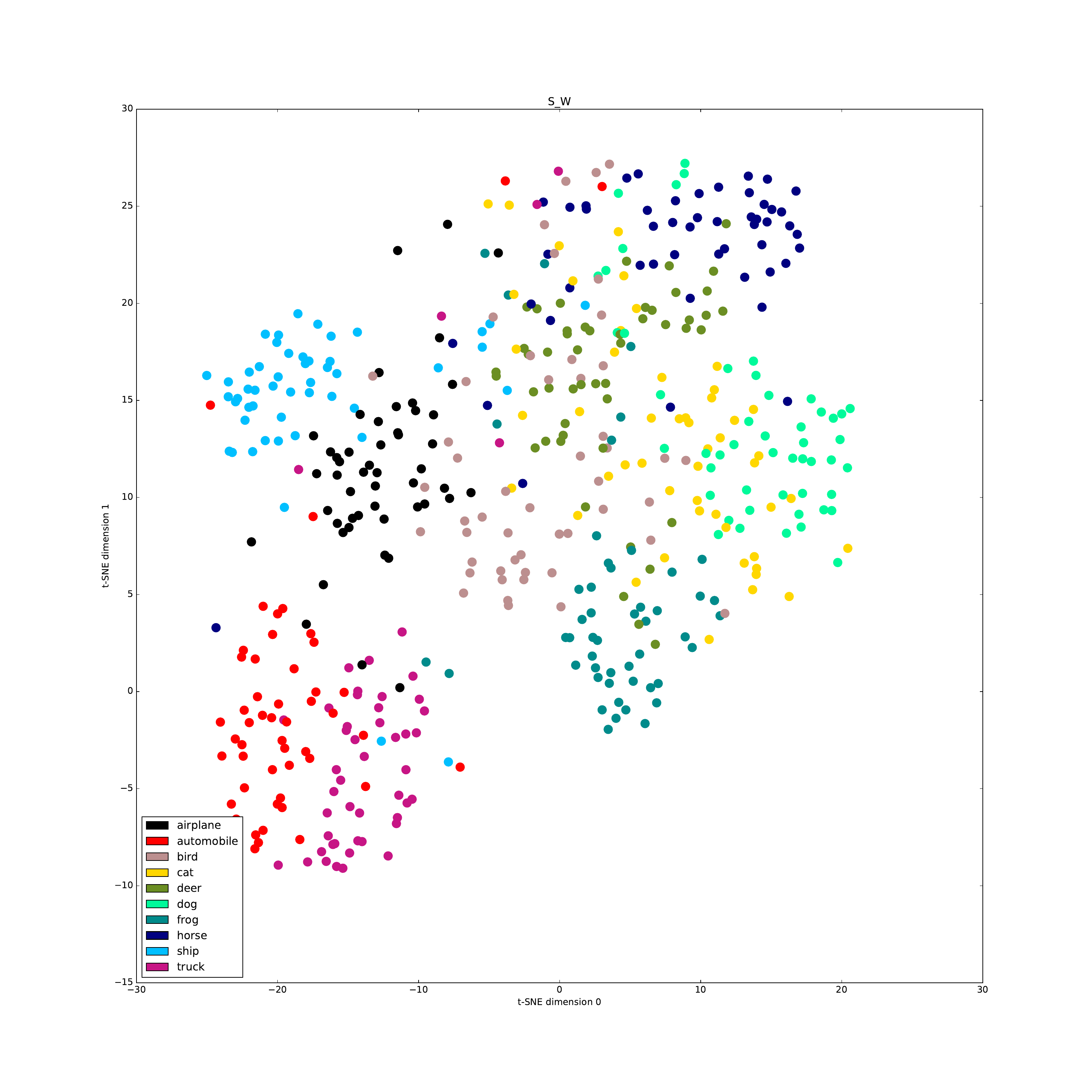}
		\caption{CIFAR-10}
		\label{fig:cifar10SW}
	\end{subfigure}
	~ 
	\begin{subfigure}[b]{0.43\textwidth}
		\centering
		\includegraphics[scale=0.095]{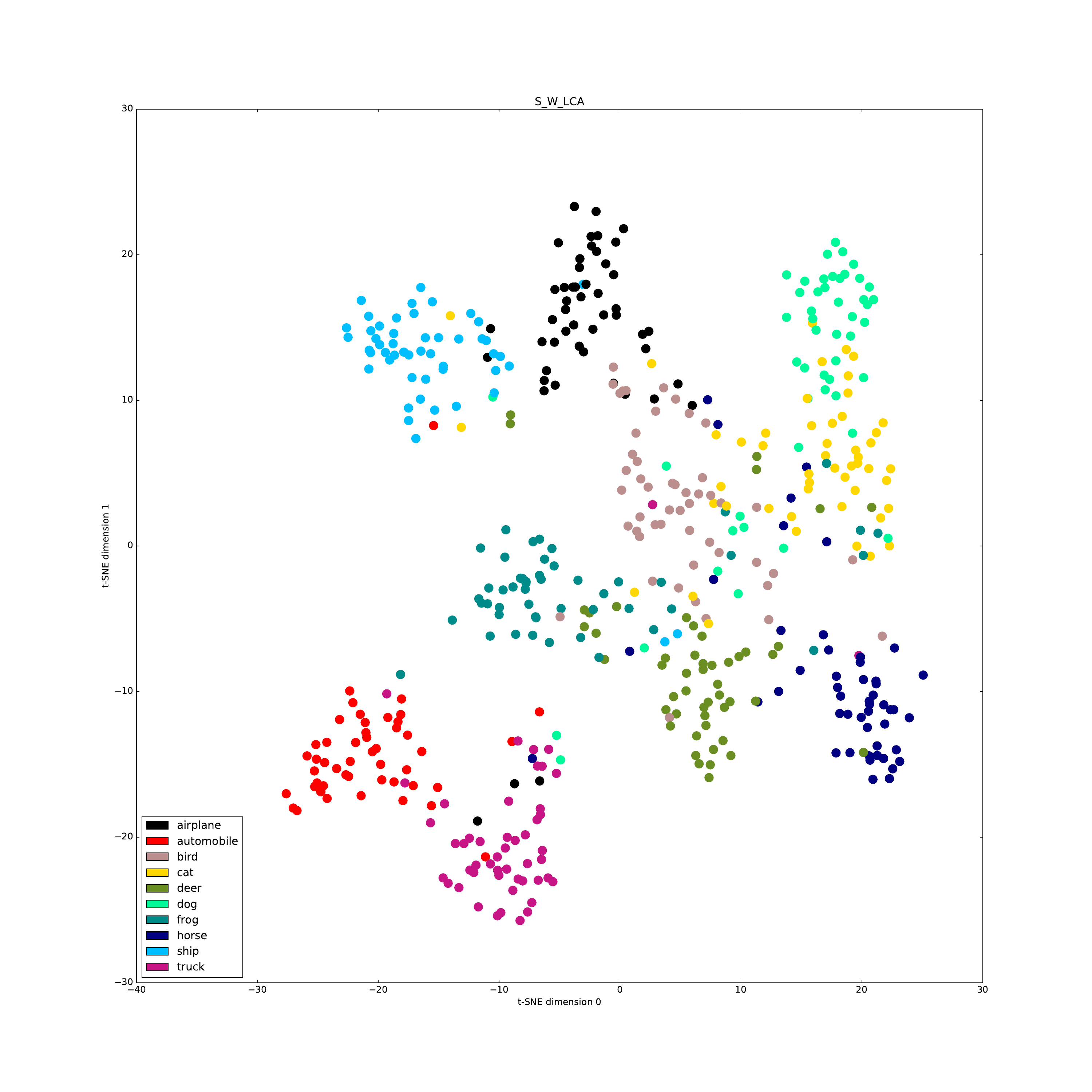}
		\caption{CIFAR-10 with LCNN}
		\label{fig:cifar10SWLCA}
	\end{subfigure}\\
	\caption{t-SNE 2D visualizations of a few samples from test set of CIFAR10 and MNIST. We see that model complexity control consistently enforces crisper, more distinct clustering of classes in feature space.}
	\label{fig:cnn_visualization}
\end{figure*}

Table \ref{tab:mnist_acc} displays the performance of the LCNN learning rule in conjunction with popular generalization techniques. We find that the application of the LCNN loss improves upon the existing methods in terms of test set accuracies. In this case difference between training and test set error, denoted by the difference column in Table \ref{tab:mnist_acc} is not appreciable for any of the methods,  hence we move our attention to a comparatively difficult dataset CIFAR-10, where we see a more distinct trend in performance gains.

\begin{table}[t]
	\centering
	\scalebox{0.7}{
		\begin{tabular}{|c|c|c|c|}
			\hline
			Methodology & \multicolumn{1}{|c|}{Train Acc} & \multicolumn{1}{|c|}{Test Acc} & \multicolumn{1}{|c|}{Difference} \\
			\hline
			S + W & 0.981 & 0.982 & -0.001 \\
			S + W + BN & 0.985 & 0.984 & 0.001 \\
			S + W + D & 0.979 & 0.982 & -0.003 \\
			S + W + D-A & 0.985 & 0.979 & 0.006 \\
			S + W + D + BN & 0.978 & 0.978 & 0.000 \\
			S + W + LC-L & 0.983 & 0.982 & 0.001 \\
			S + W + LC-A & 0.982 & 0.982 & 0.000 \\
			S + W + D + LC-A & 0.978 & 0.978 & 0.000 \\
			S + W + D+ LC-L & 0.981 & 0.980 & 0.001 \\
			S + W + BN + LC-A & 0.988 & \textbf{0.985} & 0.003 \\
			S + W + BN + LC-A + D & 0.980 & 0.979 & 0.001 \\
			\hline
		\end{tabular}%
	}
\caption{Results on MNIST with LeNet.}
\label{tab:mnist_acc}%
\end{table}%
\begin{table}[t]
	\centering
	\scalebox{0.7}{
	\begin{tabular}{|c|c|c|c|}
		\hline
		Methodology & \multicolumn{1}{|c|}{Train Acc} & \multicolumn{1}{|c|}{Test Acc} & \multicolumn{1}{|c|}{Difference} \\
		\hline
		S + W & 0.869 & 0.765 & 0.103 \\
		S + W + BN & 0.913 & 0.781 & 0.132 \\
		S + W + D  & 0.855 & 0.768 & 0.086 \\
		S + W + D-A  & 0.863 & 0.771 & 0.092 \\
		S + W + D + BN & 0.874 & 0.791 & 0.083 \\
		S + W + LC-L & 0.881 & 0.793 & 0.087 \\
		S + W + LC-A & 0.807 & \textbf{0.796} & \textbf{0.010} \\
		S + W + D + LC-L & 0.840 & 0.779 & 0.061 \\
		S + W + D + LC-A & 0.771 & 0.741 & 0.029 \\
		S + W + BN + LC-A & 0.895 & 0.780 & 0.115 \\
		S + W + BN + LC-A + D & 0.873 & 0.787 & 0.086 \\
		\hline
	\end{tabular}%
}
\caption{Results on CIFAR-10 with Caffe Quick CIFAR-10 architecture.}
\label{tab:cifar10_acc}%
\end{table}%

The results in Table \ref{tab:cifar10_acc} show that architectures with both model complexity control and weight regularization work the best - this is corroborated by the highest test set accuracies, with minimum difference between training and test error. As a qualitative experiment, we visualize 2D projections of randomly selected 50 points from each class using the test set in both the CIFAR-10 and MNIST datasets and compare the visualizations for S + W and S + W + LC-A networks. We see that latter produces cleaner clusters than former, with more distinct clusters being formed. The visualizations are shown in Figure \ref{fig:cnn_visualization}.


\subsubsection{Feedforward (Fully Connected) Neural Networks}\label{sec:FNN}
We now show the effectiveness of incorporating model complexity on vanilla fully connected single hidden layer feedforward architectures. Here, we have two LCNN paradigms - first we apply the LCNN loss only on the classifier layer and subsequently we apply it to all layers in conjunction with $L_2$ weight regularization. It is generally observed that fully connected nets have higher number of parameters than their convolutional counterparts, hence they are more prone to overfitting. It is thus imperative to include model complexity control for such datasets.

For RFNNs, the weight hyperparameter was tuned in the range $[10^{-4},1]$ in multiples of $10$, whereas the network with model complexity regularizer (LCNN) was applied on the last layer without $L_2$ regularization had its hyperparameter tuned in the range $[10^{-9},10^{-3}]$ in multiples of $10$. Algorithms were compared with regard to accuracies on the test set.

\begin{table}[t]
	\centering
		\scalebox{0.7}{
	\begin{tabular}{|l|c|c|c|c|c|}
		\hline
		Dataset & \multicolumn{1}{|l|}{Features} & \multicolumn{1}{|c|}{Classes} & \multicolumn{1}{|c|}{\# Train} & \multicolumn{1}{|c|}{\# Val} & \multicolumn{1}{|c|}{\# Test} \\
		\hline
		a9a   & 122   & 2     & 26049 & 6512  & 16281 \\
		protein & 357   & 3     & 14895 & 2871  & 6621 \\
		seismic & 50    & 3     & 63060 & 15763 & 19705 \\
		w8a   & 300   & 2     & 39800 & 9949  & 14951 \\
		webspam uni & 254   & 2    & 210000 & 70001 & 69999 \\
		\hline
	\end{tabular}%
}
	\caption{Datasets used in fully connected networks.}
	\label{tab:largeFNNdatasets}%
\end{table}%
\begin{table}[t]
	\centering
	\scalebox{0.7}{
	\begin{tabular}{|l|c|c|c|}
		\hline
		Datasets & \multicolumn{1}{|c|}{SE} & \multicolumn{1}{|c|}{SE + LC-L} & \multicolumn{1}{|c|}{SE + W + LC-A} \\
		\hline
		a9a   & 0.837 & 0.842 & \textbf{0.847} \\
		protein & 0.617 & 0.618 & \textbf{0.673} \\
		seismic & 0.736 & \textbf{0.737} & 0.718 \\
		w8a   & \textbf{0.985} & 0.981 & 0.979 \\
		webspam uni & 0.963 & \textbf{0.967} & \textbf{0.967} \\
		\hline
	\end{tabular}%
	}
	\caption{Accuracies on large datasets used in FNN experiments. We find that the LCNN term results in consistent generalization.}
	\label{tab:accFNN}%
\end{table}%
 Table \ref{tab:largeFNNdatasets} shows the large scale datasets adopted from LIBSVM website \cite{chang2011libsvm}. In Table \ref{tab:accFNN}, we compare several methodologies across 5 datasets and show that a combination of squared error along with weight regularization and LCNN model complexity term (SE + W + LC-A) has better test accuracies in 3 of the 5 datasets, showing the need of model complexity control. Finally, in subsection \ref{sec:SAE}, we show the relevance of LCNN to the unsupervised regime.

\subsubsection{Sparse Autoencoders}\label{sec:SAE}
\begin{figure}[h]
\centering
	  \begin{subfigure}[b]{0.42\textwidth}
	\centering
		\includegraphics[scale=0.2]{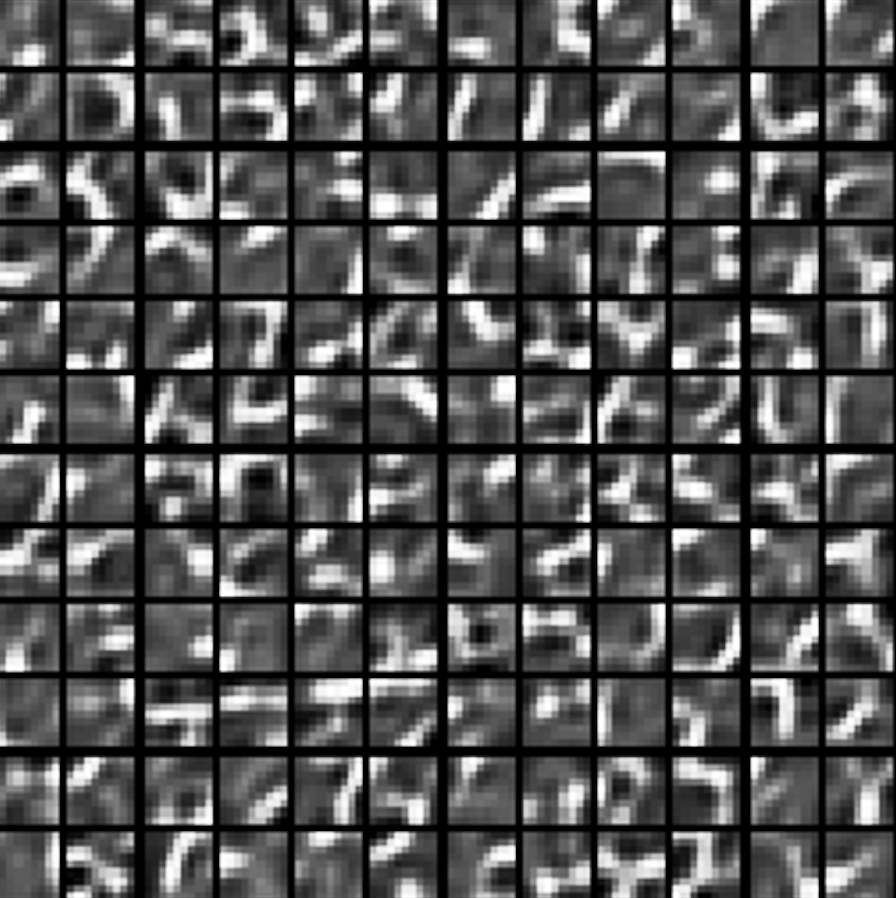}
		\caption{SAE (KL divergence only) on MNIST}
		\label{fig:SAE_filters_mnist}		
	\end{subfigure}
~
	  \begin{subfigure}[b]{0.42\textwidth}
	\centering
		\includegraphics[scale=0.2]{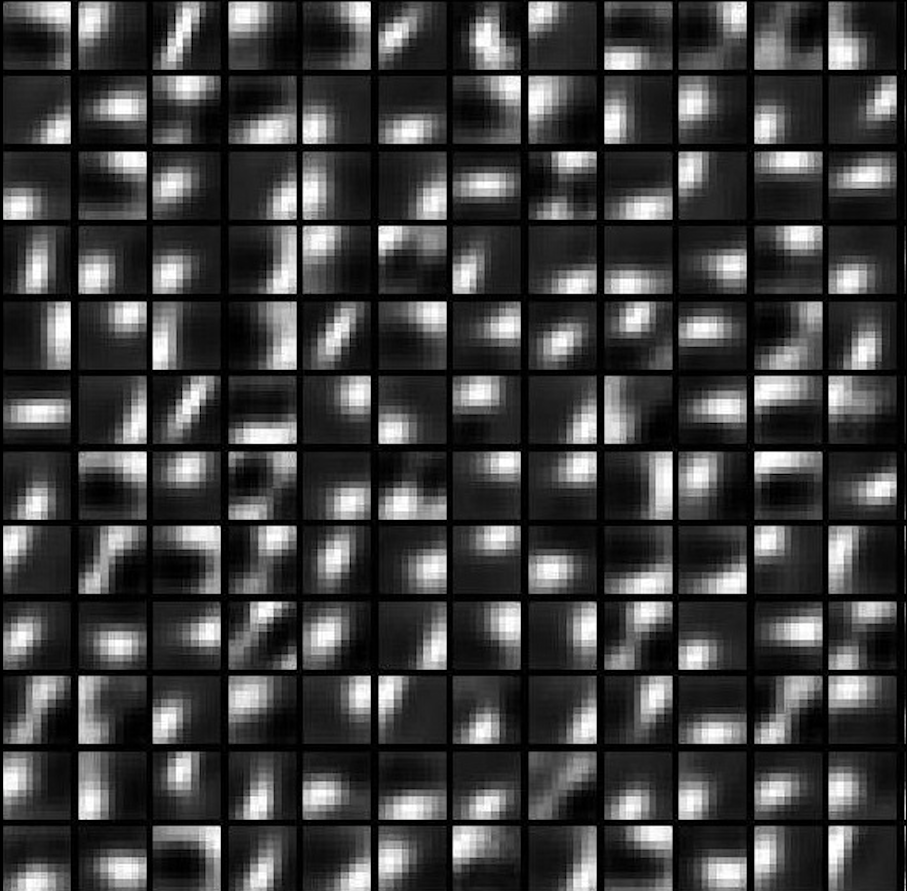}
		\caption{SAE (KL + LCNN) on MNIST}
		\label{fig:SAE_LCNN_filters_mnist}
	\end{subfigure}	
	\caption{Samples of image filters on the MNIST dataset obtained by  (a) SAE using KL divergence only, and (b) SAE using KL + LCNN. Note that filters obtained using the LCNN are visibly sharper. The Spatial-Spectral Entropy-based Quality (SSEQ) scores for (a) and (b) are 52.57 and 32.13, respectively, indicating that the LCNN filters are about 64\% superior.}
	\label{fig:mnistfilters}
\end{figure}

Sparse Autoencoders (SAE) \cite{ng2011sparse} are a popular method to learn a low dimensional manifold on which the data resides. They are an unsupervised technique to find a sparse representation in a lower dimension than that of original data, with minimum reconstruction error. It uses Kullback-Leibler (KL) divergence as a metric to force the neuronal activations close to zero, thus allowing only a small number of neurons to fire. In our experiments we compare a single layer sparse autoencoder, with LCNN loss applied at the decoder of the network along with KL divergence.

Consider a SAE, with $l$ neurons with sigmoidal activation function activations. Let $\mathbf{x}^i \in \Re^n \,\, \forall \,\, i \in \{1, \ldots, M\}$ be the input samples, let the reconstructed output be $\hat{\mathbf{x}}^i$. Let $\mathbf{w}_{e_i} \in \Re^n \,\, \forall \,\, i \in \{1, \ldots, l\}$ be the set of weights of the encoder, while the weights of the decoder are represented by $\mathbf{w}_{d_i} \,\, \in \Re^l \forall \,\, i \in \{1, \ldots, n\}$. Similarly the biases are represented as $b_{e_i}$ and $b_{d_i}$ respectively.

Let $\mathbf{u}_i \in \Re \,\, \forall \,\, i \in \{1, \ldots, l\}$ be the activations of the network. Let $\rho$ be the sparsity parameter which is kept to $0.05$ in all our experiments. Any deviation of activation from $\rho$ is penalized using KL divergence. Finally, we add a model complexity term in addition to KL divergence. The error function is then given by,
\begin{gather}
	\operatorname{Min} E = \frac{1}{2}\sum_{i=1}^{M}\|\mathbf{x}^i-\hat{\mathbf{x}}^i\|^2 + C \sum_{i=1}^{M}\sum_{j=1}^{l} KL(\rho||\mathbf{u}_j^i) + \nonumber\\
	\frac{D}{2} \sum_{i=1}^{M}\sum_{j=1}^{n} (\mathbf{w}_{d_j}^T \mathbf{u}_i + b_{d_j})^2 \label{sae_lcnn} 
\end{gather}
where,
\begin{gather}
	\hat{x_j}^i = sigmoid(w_{d_j}^T \mathbf{u}_i + b_{d_j})
\end{gather}
Finally, we compare the off the shelf SAE with KL divergence term and SAE incorporated with KL divergence and LCNN term on MNIST dataset. The number of hidden neurons in both the cases was set to 196.

It is evident from Fig. \ref{fig:mnistfilters} that image filters obtained by sparse autoencoders based on the LCNN are sharper and show a higher contrast. In order to quantify the difference, we treat the filters as images and compute the Spatial-Spectral Entropy-based Quality (SSEQ) \cite{liu2014no}. This metric measures image quality, and is statistically superior to several algorithms in the same domain. The SSEQ is at best 0 and at worst 100. For the MNIST dataset, the values of SSEQ for filters learnt by the LCNN and RFNNs are \textbf{32.13} and \textbf{52.57}, respectively, indicating that the LCNN learns crisper features.

Thus we show that the model complexity term helps in the unsupervised regime as well. We use the features learnt by the autoencoders to train classifiers. Not so surprisingly, SAE incorporated with KL divergence and the LCNN term shows better test set accuracies (\textbf{98.10\%}), than SAE with the KL divergence term alone (\textbf{97.90\%}). This indicates that the LCNN term helps in learning more relevant features. This is also supported by the plots in Figs. \ref{fig:MNIST_Hist_LCNN} and \ref{fig:MNIST_Hist_NN}, which depict the weight histograms obtained when the sparse autoencoders (SAEs) are trained on the MNIST dataset. Observe that LCNN learning tends to lead to sparser solutions (more weights being zero), but also leads to a higher peak (indicating that some relevant features may be more amplified). Many VC dimension bounds for neural networks show a rough dependence (see, for example \cite{sontag1998vc}) as $O(|W|log(|W|))$, where $|W|$ is the number of weights. The histograms indicate that the LCNN learning rule indeed leads to a reduced VC dimension bound.

\begin{figure}[hbtp]
    \centering
    \begin{subfigure}[b]{0.45\textwidth}
    \centering
        \includegraphics[scale=0.4]{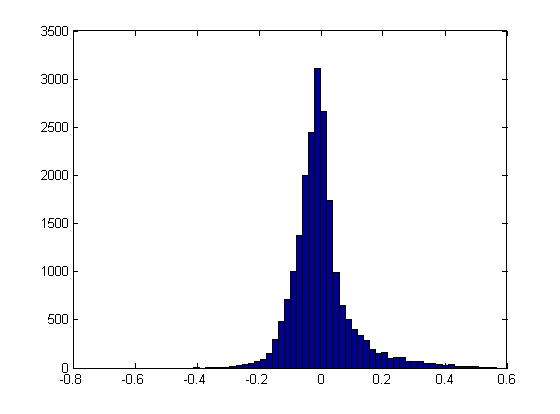}
        \caption{LCNN SAE.}
        \label{fig:MNIST_Hist_LCNN}
    \end{subfigure}
    \begin{subfigure}[b]{0.45\textwidth}
    \centering
        \includegraphics[scale=0.4]{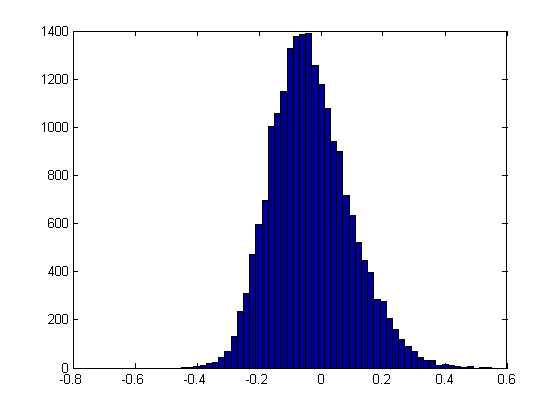}
        \caption{Regularized SAE}
        \label{fig:MNIST_Hist_NN}
    \end{subfigure}
    \caption{Histograms of weights on the MNIST dataset.}
    \label{fig:MNIST_Hist}
\end{figure}

\subsection{Gradient analysis for LCNN}

In neural networks, the gradient tends to get smaller as we move backward from the output layer, through the hidden layers. This means that neurons in the previous layers learn much more slowly than neurons in later layers. The phenomenon is known as the vanishing gradient problem \cite{bengio1994learning}. In case of gradient descent, the gradient vanishes exponentially as we traverse back through network layers. This problem is more evident in deep networks. One of the prominent reason for this is the saturating behavior of the sigmoid class of activation functions. For deriving updates for network weights and biases, the error back-propagates through activation functions. It is evident, that if we force the net activation of tansig and logsig towards zero (as the LCNN tends to do), the network will operate in stronger gradient regions, leading to faster convergence and an alleviation of the vanishing gradient problem.

This is empirically evidenced by the plots in Figs. \ref{fig:cnngradient} and \ref{fig:nngrad}. In these plots, the mean of the absolute gradient values in the final and penultimate layers has been plotted (on the primary Y-axis) along with the test error on a log scale (on the secondary Y-axis) for varying epochs of the MNIST dataset. The plots for convolutional neural networks (CNNs) with and without the LCNN loss function are illustrated in Figs. \ref{fig:grad1a} and \ref{fig:grad1b} respectively. It can be seen that the use of the LCNN error functional results in a higher value of the mean absolute gradient in the final and penultimate layer. Figs. \ref{fig:grad2a} and \ref{fig:grad2b}, respectively, show a similar trend for a conventional neural network.

\begin{figure*}[ht]
	\centering
	\begin{subfigure}[b]{0.44\textwidth}
		\centering
		\includegraphics[scale=0.38]{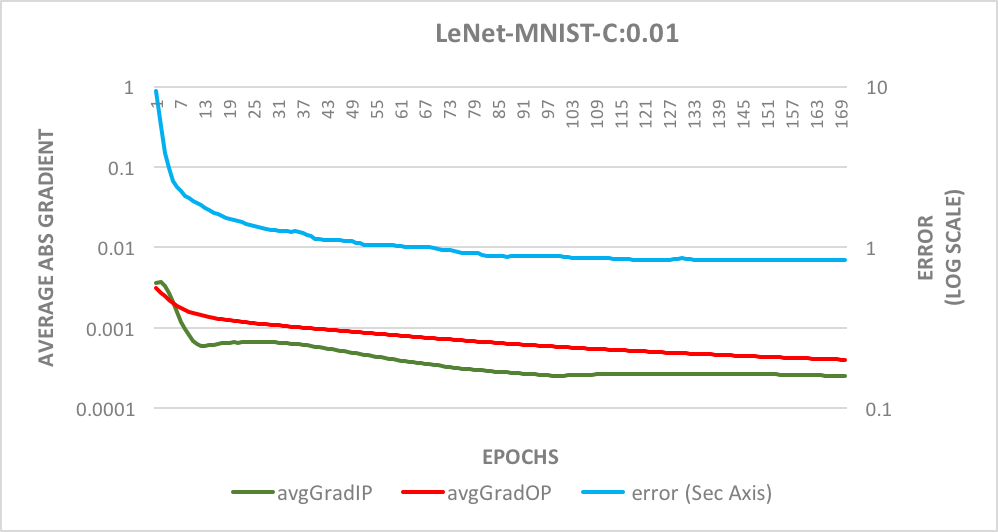}
		\caption{CNN: LeNet, LCNN C=0.01, Error after 170 epochs: 0.84}
		\label{fig:grad1a}
	\end{subfigure}
	~ 
	\begin{subfigure}[b]{0.44\textwidth}
		\centering
		\includegraphics[scale=0.38]{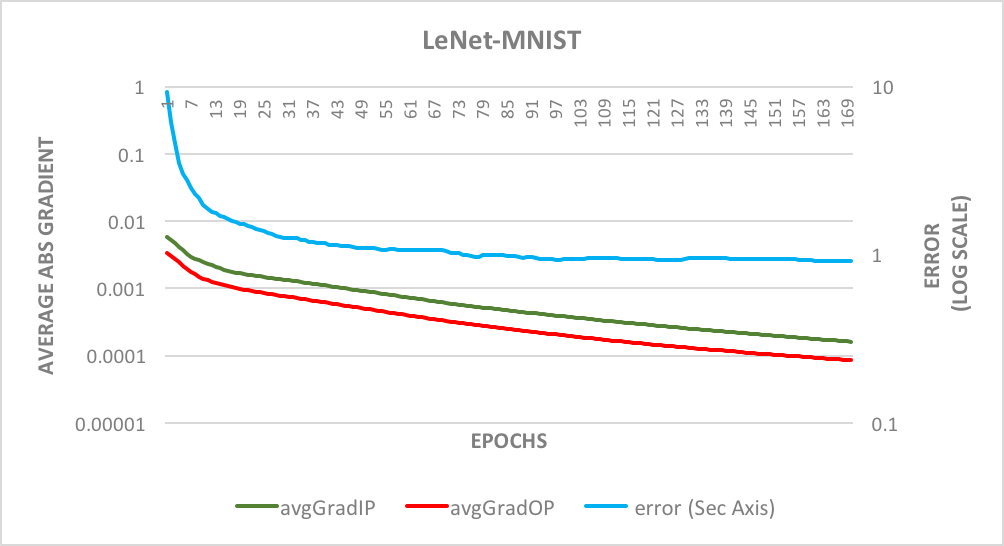}
		\caption{CNN: LeNet Architecture, Error after 170 epochs: 0.94}
		\label{fig:grad1b}
	\end{subfigure}
	\caption{Mean gradients in the last two layers of a CNN with and without the LCNN term. Note that the mean gradient in the LCNN case is almost one order of magnitude larger. At the same time, the empirical error (blue curve) is lower in the LCNN case.}
	\label{fig:cnngradient}
\end{figure*}

\begin{figure*}[ht]
	\centering
	\begin{subfigure}[b]{0.44\textwidth}
		\centering
		\includegraphics[scale=0.38]{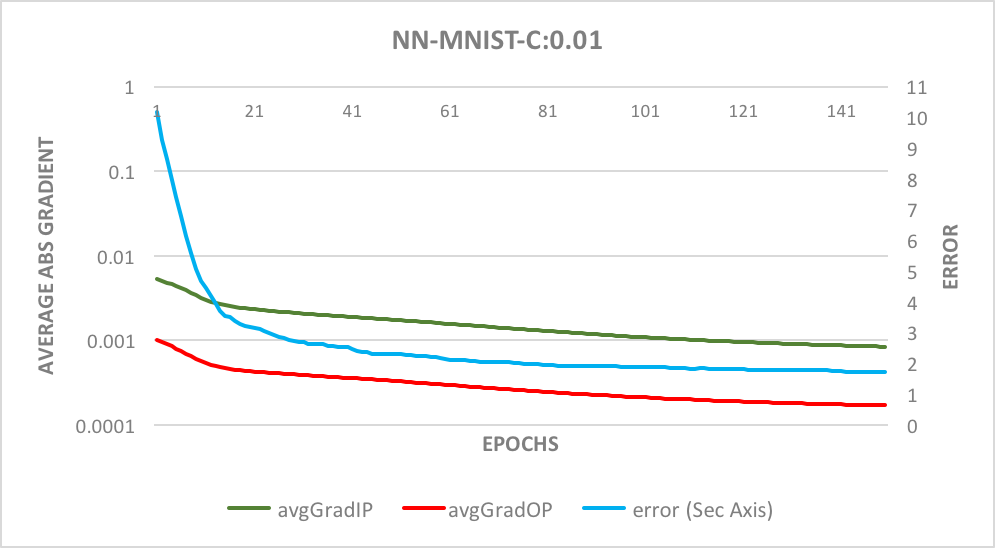}
		\caption{NN, LCNN C=0.01, Error after 150 epochs: 1.73}
		\label{fig:grad2a}
	\end{subfigure}
	~ 
	\begin{subfigure}[b]{0.44\textwidth}
		\centering
		\includegraphics[scale=0.38]{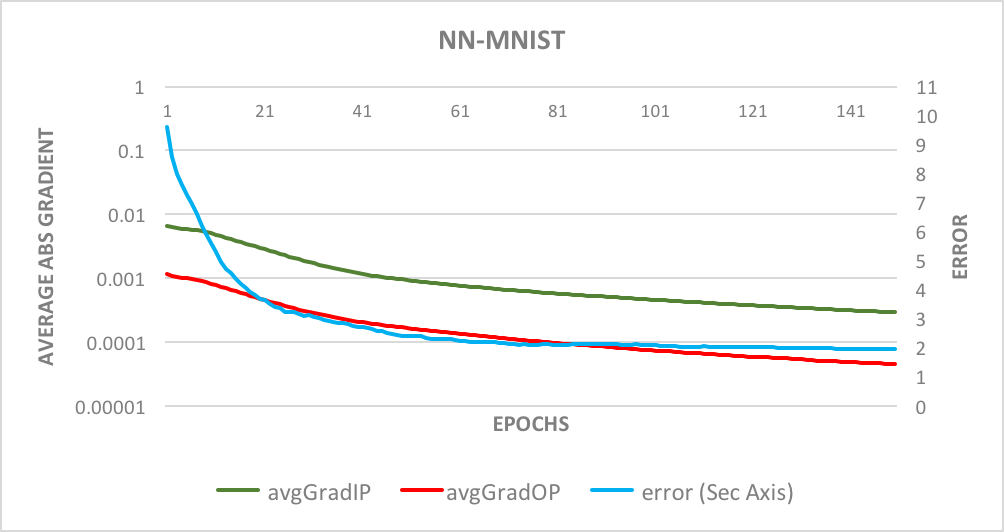}
		\caption{NN, Error after 150 epochs: 1.97}
		\label{fig:grad2b}
	\end{subfigure}
	\caption{Mean gradients in the last two layers of a a regularized neural network with and without the LCNN term. Note that the mean gradient in the LCNN case is almost one order of magnitude larger. At the same time, the empirical error (blue curve) is lower in the LCNN case.}
	\label{fig:nngrad}
\end{figure*}

\section{Conclusion and Discussion}\label{sec:conclusion}
This paper attempts to extend the ideas of minimal complexity machines \cite{jayadeva2015learning, sharma2017large}. Using these ideas, we derive a differentiable and continuous upper bound on the fat shattering VC dimension of a neural network. We define a loss function that allows us to find a tradeoff between classification error and model complexity. This loss function is a surrogate for the total risk, since it tries to minimize both the empirical error and the structural risk. The proposed Low Complexity Neural Network (LCNN) learning rule is obtained by using the gradient of the loss function and propagating it back through the layers of the network. The LCNN rule was applied to several benchmark datasets from diverse application domains, ranging from image classification to unsupervised manifold learning.

These benchmarks offer significant diversity in terms of the number of samples and the number of features. The results incontrovertibly demonstrate that the LCNN converges faster and generalizes better, than conventional feedforward neural networks with regularization. The results also show that in multiple cases the LCNN provides an edge over current regularization techniques like Dropout and Batch Normalization. We show that the LCNN classifier scales well to larger datasets. Filters learnt by the LCNN on large image datasets are sharper and show a higher contrast.

The complexity control approach presented in the paper is generic, and can be adapted to many other settings and architectures. Different algorithms for minimization of the error function can also be explored, to allow for various tradeoffs in terms of training time and storage complexity. In our experiments we use a global hyperparameter for the model complexity term. However, a more localized approach, with different hyper-parameters for each layer might offer greater flexibility. In the experiments we do not explore the area of sparsity, though it has not escaped our attention that low model complexity should yield sparse models, as in the case of the MCM  \cite{jayadeva2015learning, sharma2017large}.

\bibliographystyle{IEEEtran}
\bibliography{IEEEabrv,refs.bib}

\end{document}